%% file: sda_hal.tex
\documentclass{article}     

\usepackage[numbers]{natbib}
\usepackage[final]{neurips_2020}

\usepackage[utf8]{inputenc} 
\usepackage[T1]{fontenc}    
\usepackage{hyperref}       
\usepackage{url}            
\usepackage{booktabs}       
\usepackage{amsfonts}       
\usepackage{nicefrac}       
\usepackage{microtype}      

\usepackage[normalem]{ulem}
\useunder{\uline}{\ul}{}
\usepackage{amsmath}
\usepackage{amsfonts}
\usepackage{algorithmic}
\usepackage{graphicx}
\usepackage{tikz,mathtools}
\usepackage{amsthm}
\usepackage{caption}
\usepackage{indentfirst}
\usepackage{algorithm}
\usepackage{algorithmic}
\usepackage{dsfont}

\newtheorem*{definition}{Definition}
\newtheorem{theorem}{Theorem}[section]
\newtheorem{corollary}{Corollary}[theorem]
\newtheorem{lemma}[theorem]{Lemma}

\usepackage{bm}
\newcommand{\R}{\mathbb{R}}
\newcommand{\N}{\mathbb{N}}
\newcommand{\cA}{\mathcal{A}}
\newcommand{\cS}{\mathcal{S}}
\newcommand{\cF}{\mathcal{F}}
\newcommand{\cM}{\mathcal{M}}
\newcommand{\cI}{\mathcal{I}}

\newcommand{\cG}{\mathcal{G}}
\newcommand{\cH}{\mathcal{H}}
\newcommand{\cE}{\mathcal{E}}
\newcommand{\cB}{\mathcal{B}}
\newcommand{\cD}{\mathcal{D}}
\newcommand{\cZ}{\mathcal{Z}}
\newcommand{\cL}{\mathcal{L}}
\newcommand{\cC}{\mathcal{C}}
\newcommand{\cW}{\mathcal{W}}

\newcommand{\cX}{\mathcal{X}}
\newcommand{\cR}{\mathcal{R}}

\newcommand{\aargmax}{\text{argmax}}

\newcommand{\bP}{\mathbb{P}}
\newcommand{\bE}{\mathbb{E}}
\newcommand{\ind}{\mathds{1}}
\newcommand{\kl}{\mathrm{kl}}

\newcommand{\K}{\{1,\dots,K\}}
\newcommand{\SP}{\ensuremath{\mathrm{SP}}}

\usepackage{todonotes}

\newcommand{\RBSSDA}{\ensuremath{\mathrm{RB}\text{-}\mathrm{SDA}}}
\newcommand{\SSDA}{\ensuremath{\mathrm{SDA}}}
\newcommand{\SDA}{\ensuremath{\mathrm{SDA}}}
\newcommand{\SPSSDA}{\ensuremath{\mathrm{SP}\text{-}\mathrm{SDA}}}
\newcommand{\WRSSDA}{\ensuremath{\mathrm{WR}\text{-}\mathrm{SDA}}}
\newcommand{\LBSSDA}{\ensuremath{\mathrm{LB}\text{-}\mathrm{SDA}}}
\newcommand{\LDSSSDA}{\ensuremath{\mathrm{LDS}\text{-}\mathrm{SDA}}}
\newcommand{\sampler}{independent sampler}
\newcommand{\RB}{\ensuremath{\mathrm{RB}}}
\newcommand{\WR}{\ensuremath{\mathrm{WR}}}

\let\ln\relax
\DeclareMathOperator{\ln}{log}
\let\epsilon\relax
\DeclareMathOperator{\epsilon}{\varepsilon}

\title{Sub-sampling for Efficient Non-Parametric \\ Bandit Exploration}

\author{
  Dorian Baudry \hspace{2.3cm}  Emilie Kaufmann \hspace{2.1cm} Odalric-Ambrym Maillard 
  \\
  \texttt{dorian.baudry@inria.fr} \quad\texttt{emilie.kaufmann@univ-lille.fr} \quad 
	\texttt{odalric.maillard@inria.fr} \\ \\
 \hspace{-1.2cm} Univ. Lille, CNRS, Inria, Centrale Lille, UMR 9198-CRIStAL, F-59000 Lille, France}

\begin{document}

\maketitle

\begin{abstract}

In this paper we propose the first multi-armed bandit algorithm based on \textit{re-sampling} that achieves asymptotically optimal regret simultaneously for different families of arms (namely Bernoulli, Gaussian and Poisson distributions). Unlike Thompson Sampling which requires to specify a different prior to be optimal in each case, our proposal RB-SDA does not need any distribution-dependent tuning. RB-SDA belongs to the family of Sub-sampling Duelling Algorithms (\SDA) which combines the \emph{sub-sampling} idea first used by the BESA \cite{BESA} and SSMC \cite{SSMC} algorithms with different sub-sampling schemes. In particular, RB-SDA uses \textit{Random Block} sampling. We perform an experimental study assessing the flexibility and robustness of this promising novel approach for exploration in bandit models. 
\end{abstract}

\section{Introduction}

A $K$-armed bandit problem is a sequential decision-making problem in which a learner sequentially samples from $K$ unknown distributions called arms. In each round the learner chooses an arm $A_t \in \{1,\dots,K\}$ and obtains a random reward $X_t$ drawn from the distribution of the chosen arm, that has mean $\mu_{A_t}$. The learner should adjust her sequential sampling strategy $\cA = (A_t)_{t\in\N}$ (or bandit algorithm) in order to maximize the expected sum of rewards obtained after $T$ selections. This is equivalent to minimizing the \emph{regret}, defined as the difference between the expected total reward of an oracle strategy always selecting the arm with largest mean $\mu_\star$ and that of the algorithm:
\[\cR_T(\cA) = \mu_\star T - \bE\left[\sum_{t=1}^{T} X_t\right] = \bE\left[\sum_{t=1}^{T} (\mu_\star - \mu_{A_t})\right]. \]
An algorithm with small regret needs to balance exploration (gain information about arms that have not been sampled a lot) and exploitation (select arms that look promising based on the available information). Many approaches have been proposed to solve this exploration-exploitation dilemma (see \cite{BanditBook} for a survey), the most popular being Upper Confidence Bounds (UCB) algorithms \cite{agrawal95,auer2002finite,KL_UCB} and Thompson Sampling (TS) \cite{TS_1933,TS12AG}. TS is a randomized Bayesian algorithm that selects arms according to their posterior probability of being optimal. These algorithms enjoy logarithmic regret under some assumptions on the arms, and some of them are even \emph{asymptotically optimal} in that they attain the smallest possible asymptotic regret given by the lower bound of Lai \& Robbins \cite{LaiRobbins85}, for some parametric families of distributions. For distributions that are continuously parameterized by their means, this lower bound states that under any uniformly efficient algorithm,  
\begin{equation} \liminf_{T\rightarrow \infty} \frac{\cR_T(\cA)}{\log(T)} \geq \sum_{k : \mu_k < \mu_\star} \frac{(\mu_\star - \mu_k)}{\kl(\mu_k,\mu_\star)}.\label{eq:LRLB},\end{equation} where $\kl(\mu,\mu')$ is the Kullback-Leibler divergence between the distribution of mean $\mu$ and that of mean $\mu'$ in the considered family of distributions. 
For arms that belong to a one-parameter exponential family (e.g. Bernoulli, Gaussian, Poisson arms) kl-UCB using an \emph{appropriate divergence function} \cite{KL_UCB} and Thompson Sampling using an \emph{appropriate prior distribution} \cite{TS_Emilie,AG13Further,korda13} are both \emph{asymptotically optimal} in the sense that their regret matches that prescribed by the lower bound \eqref{eq:LRLB}, for large values of $T$. Yet, a major drawback of theses algorithms is that their optimal tuning requires the knowledge of the families of distributions they operate on. In this paper, we overcome this issue and propose an algorithm that is simultaneously asymptotically optimal for several families of distributions.

In the past years, there has been a surge of interest in the design of non-parametric algorithms that directly use the empirical distribution of the data instead of trying to fit it in an already defined model, and are therefore good candidates to meet our goal. 
In \cite{Giro}, the authors propose the General Randomized Exploration (GRE) framework in which each arm $k$ is assigned an index $\hat{\mu}_{k,t}$ sampled from a distribution $p(\cH_{k,t})$ that depends on the history of past observed rewards for this arm $\cH_{k,t}$, and the arm with largest index is selected. GRE includes Thompson Sampling (for which $p(\cH_{k,t})$ is the posterior distribution given a specified prior) but also allows for more general non-parametric re-sampling schemes. However, the authors of \cite{Giro,PHE} show that setting $p(\cH_{k,t})$ to be the non-parametric Bootstrap \cite{efron_intro_bootstrap} leads to linear regret. They propose variants called GIRO and PHE which perturb the history by augmenting it with fake samples. History perturbation was already suggested by \cite{Osband_BTS} and is also used by Reboot \cite{Reboot}, with a slightly more complicated bootstrapping scheme. Finally, the recently proposed Non Parametric TS \cite{Honda} does not use history perturbation but instead sets $\hat{\mu}_{k,t}$ as a weighted combination of all observations in $\cH_{k,t}$ and the upper bound of the support, where the weights are chosen uniformly at random in the simplex of dimension $|\cH_{k,t}|$. 

Besides Reboot \cite{Reboot}, which has been analyzed only for Gaussian distributions, all other algorithms have been analyzed for distributions with known bounded support, for which they are proved to have logarithmic regret. Among them, Non Parametric TS has strong optimality property as its regret is proved to match the lower bound of Burnetas and Katehakis \cite{burnetas96LB} for (non-parametric) distribution that are bounded in [0,1]. In this paper, we propose the first re-sampling based algorithm that is asymptotically optimal for several classes of possibly un-bounded parametric distributions. We introduce a new family of algorithms called Sub-Sampling Duelling Algorithms, and provide a regret analysis for \RBSSDA{}, an algorithm based on \emph{Random Block} sub-sampling. In Theorem~\ref{th::log_regret_sda}, we show that \RBSSDA{} has logarithmic regret under some general conditions on the arms distributions. These conditions are in particular satisfied for Gaussian, Bernoulli and Poisson distribution, for which we further prove in Corollary~\ref{cor:opt_exp} that \RBSSDA{} is asymptotically optimal.   

The general \SSDA{} framework that we introduce is inspired by two ideas first developed for the BESA algorithm by \cite{BESA} and for the SSMC algorithm by \cite{SSMC}: 1) the arms pulled are chosen according to the outcome of pairwise comparison (\emph{duels}) between arms, instead of choosing the maximum of some index computed for each arm as GRE algorithms do, and 2) the use of \textit{sub-sampling}: the algorithm penalizes arms that have been pulled a lot by making them compete with the other arms with only a fraction of their history. More  precisely, in a duel between two arms $A$ and $B$ selected $n_A$ and $n_B$ times respectively, with $n_A < n_B$, the empirical mean of arm $A$ is compared to the empirical mean of a sub-sample of size  $n_A$ of the history of arm $B$. 
In BESA the duels are organized in a tournament and only the winner is sampled, while SSMC uses rounds of $K\!-\!1$ duels between an arm called \textit{leader} and all other arms. Then the leader is pulled only if it wins all the duels, otherwise all the winning \textit{challengers} are pulled. 
Second difference is that in BESA the sub-sample of the leader's history is obtained with \textit{Sampling Without Replacement}, whereas SSMC selects this sub-sample as the block of consecutive observations with smallest empirical mean. Hence BESA uses randomization while SSMC does not. Finally, SSMC also uses some forced exploration (i.e. selects any arm drawn less than $\sqrt{\log r}$ times in round $r$). In \SSDA{}, we propose to combine the round structure for the duels used by SSMC with the use of a sub-sampling scheme assumed to be independent of the observations in the history (this generalizes the BESA duels), and we get rid of the use of forced exploration.

The rest of the paper is structured as follows. In Section~\ref{sec:SSDA} we introduce the \SSDA{} framework and present different instances that correspond to the choice of different sub-sampling algorithms, in particular \RBSSDA{}. In Section~\ref{sec:regret} we present upper bounds on the regret of \RBSSDA{}, showing in particular that the algorithm is asymptotically optimal for different exponential families. We sketch the proof of Theorem~\ref{th::log_regret_sda} in Section~\ref{sec:sketch}, highlighting two important tools: First, a new concentration lemma for random sub-samples (Lemma~\ref{concentration}). Second, an upper bound on the probability that the optimal arm is under-sampled, which decouples the properties of the sub-sampling algorithm used, and that of the arms' distributions (Lemma~\ref{lemma::decomposition}). Finally, Section~\ref{sec:expes} presents the results of an empirical study comparing several instances of \SSDA{} to asymptotically optimal parametric algorithms and other algorithms based on re-sampling or sub-sampling. These experiments reveal the robustness of the \SSDA{} approaches, which match the performance of Thompson Sampling, without exploiting the knowledge of the distribution.

\section{Sub-sampling Duelling Algorithms}\label{sec:SSDA}

In this section, we introduce the notion of Sub-sampling Duelling Algorithm (\SSDA{}). We first introduce a few notation. For every integer $n$, we let $[n] = \{1,\dots,n\}$. We denote by $(Y_{k,s})_{s\in\N}$ the i.i.d. sequence of successive rewards from arm $k$, that are i.i.d. under a distribution $\nu_k$ with mean $\mu_k$. For every finite subset $\cS$ of $\N$, we denote by $\hat Y_{k,\cS}$ the empirical mean of the observations of arm $k$ indexed by $\cS$: if $|\cS| > 1$, $\hat Y_{k,\cS} := \tfrac{1}{|\cS|}\sum_{i \in \cS}Y_{k,i}.$
We also let $\hat Y_{k,n}$ as a shorthand notation  for $\hat Y_{k,[n]}$.

\paragraph{A round-based algorithm} Unlike index policies, a \SSDA{} algorithm relies on \emph{rounds}, in which several arms can be played (at most once). In each round $r$ the learner selects a subset of arms $\mathcal{A}_r=\{k_1,..., k_{i_r} \} \subseteq \K$, and receives the rewards $\cX_r= \{Y_{k_1, N_{k_1}(r)},..., Y_{k_{i_r}, N_{k_{i_r}}(r)} \}$ associated to the chosen arms, where $N_{k}(r):=\sum_{s=1}^{r} \ind(k \in \mathcal{A}_s)$ denotes the number of times arm $k$ was selected up to round $r$. 
Letting $\hat{r}_T \leq T$ be the (random) number of rounds used by algorithm $\cA$ before the $T$-th arm selection, the regret of a round-based algorithm can be upper bounded as follows:
	\begin{eqnarray}\cR_T(\cA) &=& \bE\left[\sum_{t=1}^{T} (\mu_\star - \mu_{A_t})\right] \leq \bE\left[\sum_{s=1}^{\hat{r}_T} \sum_{k=1}^{K}(\mu_\star - \mu_{k}) \ind(k \in \cA_s)\right] \nonumber\\ 
	&\leq& \bE\left[\sum_{s=1}^{T} \sum_{k=1}^{K}(\mu_\star - \mu_{k}) \ind(k \in \cA_s)\right]  =  \sum_{k = 1}^{K}(\mu_\star - \mu_{k}) \bE\left[N_k(T)\right]\,.\label{pullstoregret}\end{eqnarray}
Hence upper bounding $\bE[N_k(T)]$ for each sub-optimal arm provides a regret upper bound. 

\paragraph{Sub-sampling Duelling Algorithms} A \SSDA{} algorithm takes as input a \emph{sub-sampling algorithm} $\mathrm{SP}(m,n,r)$ that depends on three parameters: two integers $m \geq n$ and a round $r$. A call to $\mathrm{SP}(m,n,r)$ at round $r$ produces a subset of $[m]$ that has size $n$, modeled as a random variable that is further assumed to be independent of the rewards generated from the arms, $(Y_{k,s})_{k \in [K], s \in \N^*}$. 

In the first round, a \SSDA{} algorithm selects $\cA_1 = [K]$ in order to initialize the history of all arms. For $r \geq 1$, at round $r+1$, a \SSDA{} algorithm based on a sampler \SP{}, that we refer to as \SPSSDA{}, first computes the \emph{leader}, defined as the arm being selected the most in the first $r$ round: $\ell(r) = \aargmax_{k} N_k(r)$. Ties are broken in favor of the arm with the largest mean, and if several arms share this mean then the previous leader is kept or one of these arms is chosen randomly. Then the set $\mathcal{A}_{r+1}$ is initialized to the empty set and $K-1$ \emph{duels} are performed. For each "challenger" arm $k\neq\ell(r)$, a subset $\hat\cS^r_k$ of $[N_{\ell(r)}(r)]$ of size $N_k(r)$ is obtained from $\mathrm{SP}(N_{\ell(r)}(r),N_{k}(r),r)$ and arm $k$ wins the duels if its empirical mean is larger than the empirical mean of the sub-sampled history of the leader. That is 
\[\hat{Y}_{k,N_k(r)} > \hat{Y}_{\ell(r),\hat{\cS}_k^r} \ \ \Longrightarrow \ \ \ \cA_{r+1} = \cA_{r+1} \cup \{k\}\,.\]
If the leader wins all the duels, that is if $\cA_{r+1}$ is still empty after the $K-1$ duels, we set $\cA_{r+1} = \{\ell(r)\}$. Arms in $\cA_{r+1}$ are then selected by the learner in a random order and are pulled if the total budget of pulls remains smaller than $T$. The pseudo-code of \SPSSDA{} is given in Algorithm~\ref{alg::algo_1}. 

\begin{algorithm}[t]
	\caption{\SPSSDA{}\label{alg::algo_1}}
	{\small
	\begin{algorithmic}
		\REQUIRE K arms, horizon T, Sampler \SP{}
		\STATE $t\leftarrow K$, $r\leftarrow 1$, $\forall k, N_k\leftarrow 1, \cH_k\leftarrow \{Y_{k, 1}\}$ (Each arm is drawn once)
		\WHILE{$t<T$}
		\STATE $r \leftarrow r+1$, $\cA \leftarrow \{ \}$, $\ell \leftarrow \text{leader}(N, \cH, \ell)$ (Initialize the round)
		\FOR{$k \neq \ell \in 1,...,K$}
		\STATE Draw $\hat S_k^r \sim \text{SP}(N_\ell, N_k,r)$ (Choice of the sub-sample of $\ell$ used for the duel with $k$)
		\IF{$\hat{Y}_{k, N_k} > \hat Y_{\ell, \hat{S}_{k}^{r}}$}
		\STATE $\cA \leftarrow \cA \cup \{k\}$ (Duel outcome)
		\ENDIF
		\ENDFOR
		\IF{$|\cA|=0$}
		\STATE $\cA \leftarrow \{\ell\}$ 
		\ENDIF 
		\IF{$|\cA|>T-t$}
		\STATE $\cA \leftarrow \text{choose}(\cA, T-t)$ (Randomly selects a number of arm that does not exceed the budget)
		\ENDIF 
		\FOR{$a \in \cA$}
		\STATE Pull arm $a$, observe reward $Y_{a, N_a+1}$ 
		\STATE $t \leftarrow t+1$, $N_a \leftarrow N_a+1$, $\cH_a \leftarrow \cH_a \cup \{Y_{a, N_a} \}$ (Update step)
		\ENDFOR
		\ENDWHILE
	\end{algorithmic}}
\end{algorithm}

To properly define the random variable $\hat{S}^r_k$ used in the algorithm, we introduce the following probabilistic modeling: for each round $r$, each arm $k$, we define a family $(S_k^r(m,n))_{m \geq n}$ of independent random variables such that $S_k^r(m,n) \sim \text{SP}(m,n,r)$. In words, $S_k^r(m,n)$ is the subset of the leader history used  should arm $k$ be a challenger drawn $n$ times up to round $r$ dueling against a leader that has been drawn $m$ times. With this notation, for each arm $k \neq \ell(r)$ one has  $\hat{S}_k^r = S_k^r\left(N_{\ell(r)}(r),N_k(r), r\right)$. We recall that in the \SSDA{} framework, it is crucial that those random variables are independent from the reward streams $(Y_{k,s})$ of all arms $k$. We call such sub-sampling algorithms \textit{\sampler}.

\paragraph{Particular instances} We now present a few sub-sampling algorithms that we believe are interesting to use within the \SSDA{} framework. Intuitively, these algorithms should ensure enough \emph{diversity} in the output subsets when called in different rounds, so that the leader cannot always look good, and challengers may win and be explored from time to time. The most intuitive candidates are random samplers like \textit{Sampling Without Replacement} (WR) and \textit{Random Block Sampling} (RB): the first one returns a subset of size $n$ selected uniformly at random in $[m]$, while the second draws an element $n_0$ uniformly at random in $[m-n]$ and returns $\{n_0+1,...,n_0+n\}$. But we also propose two deterministic sub-sampling: \textit{Last Block} (LB) which returns $\{m-n+1,...,m\}$, and \textit{Low Discrepancy Sampling} (LDS) that is similar to RB with the first element $n_0$ of the block at a round $r$ defined as $\lceil u_r (m-n) \rceil$ with $u_r$ a predefined low discrepancy sequence \cite{drmota_discrepancy} (Halton \cite{HALTON1964}, Sobol \cite{SOBOL1967}). We believe that these last two samplers may ensure enough diversity without the need for random sampling. These four variants of \SSDA{} will be compared in Section~\ref{sec:expes} in terms of empirical efficiency and numerical complexity. For \RBSSDA{}, we provide a regret analysis in the next sections, highlighting what parts may or may not be extended to other sampling algorithms.

\paragraph{Links with existing algorithms} The BESA algorithm \cite{BESA} with $K=2$ coincides with \WRSSDA{}. However beyond $K>2$, the authors of \cite{BESA} rather suggest a tournament approach, without giving a regret analysis. \WRSSDA{} can therefore be seen as an alternative generalization of BESA beyond 2 arms, which performs much better than the tournament, as can be seen in Section~\ref{sec:expes}. While the structure of SSDA is close to that of SSMC \cite{SSMC}, SSMC is not a \SPSSDA{} algorithm, as its sub-sampling algorithm heavily relies on the rewards, and is therefore not an \sampler. Indeed, it outputs the set $\cS=\{n_0+1, \dots,n_0+ n\}$ for which $\hat{Y}_{\ell(r),\cS}$ is the smallest. The philosophy of SSMC is a bit different than that of SSDA: while the former tries to disadvantage the leader as much as possible, the latter only tries to make the leader use different parts of its history. Our experiments reveal that the SSMC approach seems to lead to a slightly larger regret, due to a bit more exploration in the beginning. Finally, we emphasize that alternative algorithms based on re-sampling (PHE, Reboot, Non-Parametric TS) are fundamentally different to \SSDA{} as they do not perform \emph{sub-sampling}. 

\paragraph{On the use of forced exploration} In\cite{SSMC}, SSMC additionally requires some \textit{forced exploration}: each arm $k$ such that $N_k(r)$ is smaller than some value $f_r$ is added to $\cA_{r+1}$. SSMC is proved to be asymptotically optimal for exponential families provided that $f_r=o(\log r)$ and $\log \log r = o(f_r)$. In the next section, we show that \RBSSDA{} does not need forced exploration to be asymptotically optimal for Bernoulli, Gaussian and Poisson distributions. However, we show in Appendix~\ref{app::forced_explo} that adding forced exploration to \RBSSDA{} is sufficient to prove its optimality for any exponential family.

\section{Regret Upper Bounds for \RBSSDA}\label{sec:regret}

In this section, we present upper bounds on the expected number of selections of each sub-optimal arm $k$, $\bE\left[N_k(T)\right]$, for the \RBSSDA{} algorithm. They directly yield an upper bound on the regret via~\eqref{pullstoregret}. To ease the presentation, we assume that there is a unique optimal arm\footnote{as can be seen in the analysis of SSMC \cite{SSMC}, treating the general case only requires some additional notation.}, and denote it by $1$. 

In Theorem~\ref{th::log_regret_sda}, we first identify some conditions on the arms distribution under which \RBSSDA{} has a regret that is provably logarithmic in $T$. In order to introduce these conditions, we recall the definition of the following \textit{balance function}, first introduced by \cite{BESA}. $\alpha_k(M,j)$ is equal to the probability that arm $1$ loses a certain amount $M$ of successive duels against $M$ sub-samples from arm $k$ that have non-overlapping support, when arm $1$ has been sampled $j$ times. 

\begin{definition}\label{def:balance}Letting $\nu_{k,j}$ denote the distribution of the sum of $j$ independent variables drawn from $\nu_k$, and $F_{\nu_{k,j}}$ its corresponding CDF, the balance function of arm $k$ is 

\vspace{-0.3cm}

\[\alpha_k(M,j) = \bE_{X \sim \nu_{1, j}}\left(\left(1-F_{\nu_{k,j}}(X)\right)^M \right).\]
\end{definition}

\begin{theorem}[Logarithmic Regret for \RBSSDA]
	\label{th::log_regret_sda}
	If the arms distributions $\nu_{1},\dots,\nu_{k}$ are such that 
	\begin{enumerate}
	 \item the empirical mean of each arm $k$ has exponential concentration given by a certain rate function $I_k(x)$ which is continuous and satisfies $I_k(x) = 0$ if and only if $x=\mu_k$:
	 \begin{equation*}
\forall x > \mu_k, \bP\left(\hat Y_{k,n}\geq x\right) \leq e^{-nI_k(x)} \ \text{and } \ \forall x < \mu_k, \bP\left(\hat Y_{k,n} \leq x\right) \leq e^{-nI_k(x)}\;, 
\end{equation*}
	 \item the balance function of each sub-optimal arm $k$ satisfies \[\forall \beta \in (0, 1), \ \ \sum_{t=1}^T \sum_{j=1}^{\lfloor(\log t)^2 \rfloor} \alpha_k(\left\lfloor \beta t/(\log t)^2 \right\rfloor, j) = o(\log T)\;.\]
	\end{enumerate}
Then, for all sub-optimal arm $k$, for all $\varepsilon > 0$, under \RBSSDA{} 

\vspace{-0.3cm}

\[\bE[N_k(T)] \leq \frac{1+\varepsilon}{I_k(\mu_1)} \log(T) + o(\log T)\;.\]
\end{theorem}

If the distributions belong to the same one-dimensional exponential family (see e.g. \cite{KL_UCB} for a presentation of some of their important properties), the Chernoff inequality tells us that the concentration condition 1. is satisfied with a rate function equal to $I_k(x) = \kl(x,\mu_k)$ where $\kl(x,y)$ is the Kullback-Leibler divergence between the distribution of mean $x$ and the distribution of mean $y$ in that exponential family. In Appendix~\ref{app::balance}, we prove that Gaussian distribution with known variance, Bernoulli and Poisson distribution also satisfy the balance condition 2., which yields the following.

\begin{corollary}\label{cor:opt_exp} Assume that the distribution of all arms belong to the family of Gaussian distributions with a known variance, Bernoulli or Poisson distributions. Then under \RBSSDA{} for all $\epsilon > 0$, for all sub-optimal arm $k$, 

\vspace{-0.4cm}

\[\bE[N_k(T)] \leq \frac{1+\varepsilon}{\kl(\mu_k,\mu_1)} \ln(T) + o_{\varepsilon,\bm\mu}(\log(T)).\]
\end{corollary}

Corollary~\ref{cor:opt_exp} permits to prove that $\limsup_{T\rightarrow} \tfrac{\cR_T(\RBSSDA)}{\ln(T)} \leq \sum_{k=2}^{K} \frac{(\mu_1-\mu_k)}{\kl(\mu_k,\mu_1)}$, which is matching the Lai \& Robbins lower bound \eqref{eq:LRLB} in each of these exponential families. In particular, \RBSSDA{} is simultaneously asymptotically optimal for different examples of bounded (Bernoulli) and un-bounded (Poisson, Gaussian) distributions. In contrast, Non-Parametric TS is asymptotically optimal for any bounded distributions, but cannot be used for Gaussian or Poisson distributions. Note that the guarantees of Corollary~\ref{cor:opt_exp} also hold for the SSMC algorithm \cite{SSMC}, but we prove that \RBSSDA{} can be asymptotically optimal \emph{without forced exploration} for some distributions. Moreover, as will be seen in Section~\ref{sec:expes}, algorithms based on randomized history-independent sub-sampling such as \RBSSDA{} tend to perform better than deterministic algorithms such as SSMC. 

Theorem~\ref{th::log_regret_sda} also shows that \RBSSDA{} may have logarithmic regret for a wider range of distributions. For example, we conjecture that a truncated Gaussian distribution also satisfy the balance condition 2.. On the other hand, condition 2. does not hold for Exponential distribution, as discussed in Appendix~\ref{app:balance_fe_exp_dist}. But we show in Appendix~\ref{app::balance_fe} that any distribution that belongs to a one-dimensional exponential family satisfies a slightly modified version of this condition, which permits to establish the asymptotic optimality of a variant of RB-SDA using forced exploration.

Finally, we note that it is possible to build on RB-SDA to propose a bandit algorithm that has logarithmic regret for any distribution that is bounded in $[0,1]$. To do so, we can use the binarization trick already proposed by \cite{AG13Further} for Thompson Sampling, and run RB-SDA on top of a binarized history $\cH_k'$ for each arm $k$ in which a reward $Y_{k,s}$ is replaced by a binary pseudo-reward is $Y_{k,s}'$ generated from a Bernoulli distribution with mean $Y_{k,s}$. The resulting algorithm inherits the regret guarantees of RB-SDA applied to Bernoulli distributions.

Characterizing the set of distributions for which the vanilla RB-SDA algorithm has logarithmic regret (without forced exploration or a binarization trick) is left as an interesting future work.

\section{Sketch of Proof}\label{sec:sketch}

In this section, we provide elements of proof for Theorem~\ref{th::log_regret_sda}, postponing the proof of some lemmas to the appendix. The first step is given by the following lemma, which is proved in Appendix~\ref{app::lem_dec}. 

\begin{lemma}\label{lem:dec} Under condition 1., for any SP-SSDA algorithm (using an \sampler), for every $\varepsilon > 0$, there exists a constant $C_{k}(\bm\nu,\epsilon)$ with $\bm\nu = (\nu_1,\dots,\nu_k)$ such that
\[\bE[N_k(T)] \leq \frac{1+\epsilon}{I_1(\mu_k)}\log(T) + 32\sum_{r=1}^{T}\bP\left(N_1(r) \leq (\log(r))^2\right) + C_{k}(\bm\nu,\epsilon)\;.\]
\end{lemma}

The proof of this result follows essentially the same decomposition as the one proposed by \cite{SSMC} for the analysis of SSMC. However, it departs from this analysis in two significant ways. First, instead of using properties of forced exploration (that is absent in \RBSSDA), we distinguish whether or not arm 1 has been selected a lot, which yields the middle term in the upper bound. Then, the argument relies on a new concentration result for sub-samples averages, that we state below. Lemma~\ref{concentration}, proved in Appendix~\ref{app::lemma_concentration}, crucially exploits the fact that a \SPSSDA{} algorithm is based on an \sampler. Using condition 1. allows to further upper bound the right-hand side of the two inequalities in Lemma~\ref{concentration} by terms that decay exponentially and contribute to the constant $C_{k}(\bm\nu,\epsilon)$.

\begin{lemma}[concentration of a sub-sample]
	\label{concentration}
	For all $(a, b)$ such that $\mu_a<\mu_b \text{, for all } \xi \in (\mu_{a} , \mu_b) \text{ and }  n_0 \in \mathbb{N}$, under any instance of \SPSSDA{} using an \sampler, it holds that {\small
		\begin{align*}
		&\sum_{s=1}^r \bP\! \left(\hat Y_{{a},N_{a}(s)}\! \geq \! \hat Y_{b, \hat \cS_b^s(N_{b}(s),N_{a}(s))}, N_b(s)\!\geq \! N_{a}(s), N_{a}(s)\!\geq\! n_0 \right)  \! \leq \! \sum_{j=n_0}^r \!\! \bP \! \left(\hat Y_{{a},j} \geq \xi \right) + r \!\!\sum_{j=n_0}^r\!\! \bP \! \left(Y_{b, j} \leq \xi
		\right), \\ 
		&\sum_{s=1}^r \bP\! \left(\hat Y_{{b},N_{b}(s)} \!\leq\! \hat Y_{a, \hat \cS_a^s(N_{a}(s),N_{b}(s))}, N_a(s)\!\geq\! N_{b}(s), N_{b}(s)\!\geq\! n_0 \right) \! \leq\! \sum_{j=n_0}^r\!\!  \bP\! \left(\hat Y_{{b},j} \leq \xi \right) + r \!\!\sum_{j=n_0}^r\!\! \bP \!\left(Y_{a, j} \geq \xi
		\right).
		\end{align*}}
\end{lemma}
So far, we note that the analysis has \emph{not} been specific to \RBSSDA{} but applies to any instance of \SSDA{}. Then, we provide in Lemma~\ref{lemma::decomposition} an upper bound on $\sum_{t=1}^T \bP\left(N_1(t)\leq (\log t)^2 \right)$ which is specific to \RBSSDA{}. This sampler is randomized and independent of $r$, hence we use the notation $\RB(m,n)=\RB(m,n,r)$. The strength of this upper bound is that it decouples the properties of the sub-sampling  algorithm and that of the arm distributions (through the balance function $\alpha_k$). 

\begin{lemma}Let $X_{m,H,j}$ be a random variable giving the   
		number of non-overlapping sub-samples of size $j$  obtained in $m$  i.i.d. samples from $\RB(H,j)$ and define $c_r=\lfloor \tfrac{r/(\log r)^2-1}{2K}\rfloor-1$. There exists $\gamma \in (0,1)$ and a constant $r_K$ such that with $\beta_{r,j} = \left\lfloor \gamma {r}/{j(\log r)^2}\right\rfloor$,		
	{\footnotesize\label{lemma::decomposition}
	$$
	\sum_{r=1}^T \bP(N_1(r)\leq \left( \log r\right)^2)  \leq r_K + \sum_{r=r_K}^T \sum_{j=1}^{\left\lfloor\log r^2\right\rfloor}\left[(K-1) \bP \left(X_{c_r, c_r, j} < \beta_{r,j} \right) + \sum_{k=2}^K \alpha_k\left(\beta_{r,j}, j\right)\right]\;. 
	$$}
\end{lemma}

To prove Lemma~\ref{lemma::decomposition} (see Appendix~\ref{app::lem_decomposition}), we extend the proof technique introduced by \cite{BESA} for the analysis of BESA to handle more than 2 arms. The rationale is that if $N_1(r) \leq \left(\log r \right)^2$ then arm $1$ is not the leader and has lost "many" duels, more precisely \textit{at least} a number of \textit{successive duels} proportional to $r/\left( \log r \right)^2$. A fraction of these duels necessarily involves sub-samples of the leader history that have non-overlapping support. Exploiting the independence of these sub-samples brings in the balance function $\alpha_k$.

In order to conclude the proof, it remains to upper bound the right hand side of Lemma~\ref{lemma::decomposition}. Using condition 2. of balanced distributions the terms depending on $\alpha_k$ sum in $o(\log T)$ and negligibly contribute to the regret. Upper bounding the term featuring $X_{m,H,j}$ amounts to establishing the following diversity property of the random block sampler.

\begin{definition}[Diversity Property]\label{def::diversity} Let $X_{m,H,j}$ be the random variable defined in Lemma~\ref{lemma::decomposition} for a randomized sampler $\SP$. $\SP$ satisfies the Diversity Property for a sequence $N_r$ of integers if

\vspace{-0.5cm}

	$$\sum_{r=1}^T \sum_{j=1}^{(\log r)^2} \bP\left(X_{N_r, N_r, j}<\gamma r/(\log r)^2 \right) = o(\log T)\,.
	$$ 
\end{definition}

We prove in Appendix~\ref{app::diversity_rbsda} that the \RB{} sampler satisfies the diversity property for the sequence $c_r$, which leads to $\sum_{t=1}^T \bP\left(N_1(t)\leq (\log t)^2 \right) = o(\ln(T))$ and concludes the proof of Theorem~\ref{th::log_regret_sda}.  

We believe that the \WR{} sampler also satisfies the diversity property (as conjectured by \cite{BESA}). While Lemma~\ref{lemma::decomposition} should apply to \WRSSDA{} as well, a different path has to be found for analyzing the \LDSSSDA{} and \LBSSDA{} algorithms, that are based on deterministic samplers and also perform well in practice. This is left for future work.

\section{Experiments} \label{sec:expes}

In this section, we perform experiments on simulated data in order to illustrate the good performance of the four instances of \SDA{} algorithms introduced in Section~\ref{sec:SSDA} for various distributions. The Python code used to perform these experiments is available on \href{https://github.com/DBaudry/Sub-Sampling-Dueling-Algorithms-Neurips20}{Github}.
 
 \paragraph{Exponential families}
 First, in order to illustrate Corollary~\ref{cor:opt_exp}, we investigate the performance of \RBSSDA{} for both Bernoulli and Gaussian distributions (with known variance 1). Our first objective is to check that for a finite horizon the regret of \RBSSDA{} is comparable with the regret of Thompson Sampling (with respectively a beta and improper uniform prior), which efficiently uses the knowledge of the distribution. Our second objective is to empirically compare different variants of \SSDA{} to other non-parametric approaches based on sub-sampling (BESA, SSMC) or on re-sampling. For Bernoulli and Gaussian distribution, Non-Parameteric TS coincides with Thompson Sampling, so we focus our study on algorithms based on history perturbation. We experiment with PHE \cite{PHE} for Bernoulli bandits and ReBoot \cite{Reboot} for Gaussian bandits, as those two algorithms are guaranteed to have logarithmic regret in each of these settings. As advised by the authors, we use a parameter $a=1.1$ for PHE and $\sigma=1.5$ for ReBoot. 

We ran experiments on 4 different Bernoulli bandit models: 1) $K=2$, $\mu=[0.8, 0.9]$, 2) $K=2$, $\mu=[0.5, 0.6]$, 3) $K=10$, $\mu_1=0.1, \mu_{2,3,4}=0.01, \mu_{5,6,7}=0.03, \mu_{8,9,10}=0.05$, 4) $K=8$ $\mu=[0.9, 0.85,\dots, 0.85]$ and 3 different bandits models with $\mathcal{N}(\mu_k, 1)$ arms with means: 1) $K=2$ $\mu=[0.5, 0]$, 2) $K=4$, $\mu=[0.5,0,0,0]$, 3) $K=4$, $\mu=[1.5, 1, 0.5, 0]$. For each experiment, Table~\ref{tab:Bernoulli} and Table~\ref{tab:Gaussian} report an estimate of the regret at time $T=20000$ based on $5000$ independent runs (extended tables with standard deviations can be found in Appendix~\ref{app::add_xp_BGTG}). The best performing algorithms are highlighted in bold. In Figure~\ref{fig:Bernoulli} and Figure~\ref{fig:Gaussian} we plot the regret of several algorithms as a function of time (in log scale) for $t \in [15000 ; 20000]$ for one Bernoulli and one Gaussian experiment respectively. We also add the Lai and Robbins lower bound $t \mapsto \left[\sum_k(\mu^\star - \mu_k)/\kl(\mu_k,\mu_\star)\right]\ln(t)$.

\vspace{-0.2cm}

\begin{table}[H]
	
	\begin{minipage}[c]{0.5\linewidth}
				\caption{Regret at $T=20000$ for Bernoulli arms \label{tab:Bernoulli}}
		\label{bernoulli-xp}
		\centering
		\addtolength{\tabcolsep}{-4pt}
		\begin{tabular}{l|l|l|l|l|l|l|l|l}
			\toprule
			& \multicolumn{4}{c|}{Benchmark} & \multicolumn{4}{c}{SDA} \\
			\cmidrule(r){2-9}
			xp      & TS & PHE & BESA & SSMC & RB& WR & LB & LDS \\
			\midrule
			1      & \textbf{11.2}  & 25.9 & \textbf{11.7}& \textbf{12.3}& \textbf{11.5} & \textbf{11.6} & \textbf{12.2}& \textbf{11.4} \\
			2      & \textbf{22.9} & \textbf{24.0} &\textbf{22.1} & \textbf{24.3} & \textbf{22.0} & \textbf{21.5} & \textbf{24.0} & \textbf{21.8} \\
			3      & \textbf{94.2} &  248.1 & \textbf{88.1} & 100.1 & \textbf{89.0} & \textbf{86.9} & 100.7 & \textbf{89.2} \\
			4      & \textbf{108.1} & 216.5 & 147.5 & 119.9 & \textbf{105.1} & \textbf{106.9} & 119.6 & \textbf{106.8}\\
			\bottomrule
		\end{tabular}
		\label{tab::bern_table}
		
	\end{minipage}
	\hfill
	\begin{minipage}[c]{0.4\linewidth}
		\captionof{figure}{Regret as a function of time for Bernoulli experiment 3\label{fig:Bernoulli}}
		\centering
		\includegraphics[width=6cm, height=3.2cm]{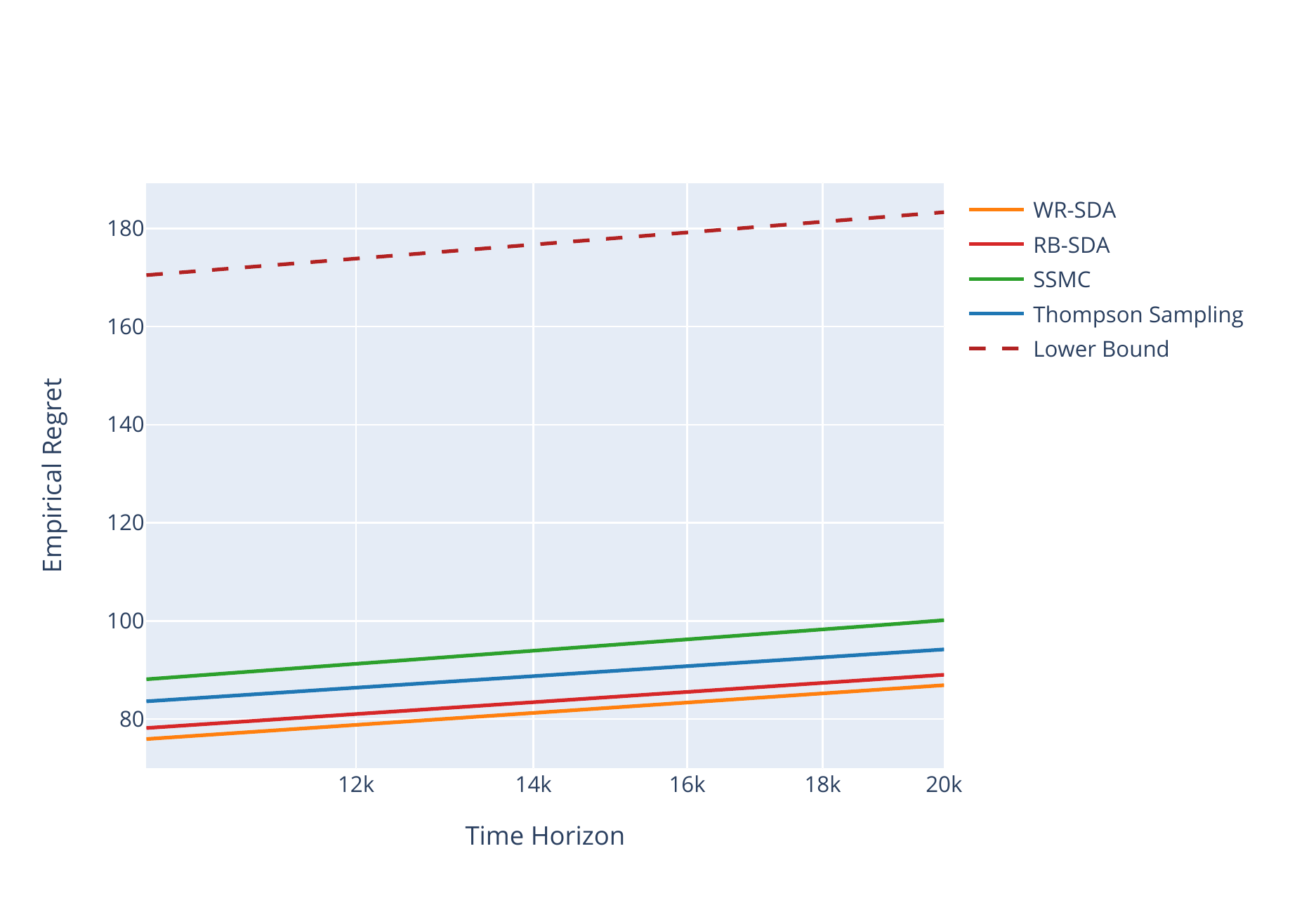}     
	\end{minipage}
	\label{fig:ma_fig}

	\vspace{-0.3cm}
\end{table}

\begin{table}[H]
	
	\begin{minipage}[c]{0.5\linewidth}
		\caption{Regret at $T=20000$ for Gaussian arms  \label{tab:Gaussian}}
		\label{gaussian-xp}
		\centering
		\addtolength{\tabcolsep}{-4pt}
		\begin{tabular}{l|l|l|l|l|l|l|l|l}
			\toprule
			& \multicolumn{4}{c|}{Benchmark} & \multicolumn{4}{c}{SDA} \\
			\cmidrule(r){2-9}
			xp      & TS     & ReBoot & BESA & SSMC & RB& WR & LB & LDS \\
			\midrule
			1      & \textbf{24.4} &  92.2& \textbf{25.3} & \textbf{26.9} & \textbf{25.6} & \textbf{24.7} & \textbf{25.1} & \textbf{26.5} \\
			2& \textbf{73.5} & 277.1&122.5 & \textbf{74.8} & \textbf{71.0} & \textbf{71.1} & \textbf{74.6} & \textbf{69.0} \\
			3 & \textbf{49.7} & 190.9&72.1 & \textbf{51.3} & \textbf{50.4} & \textbf{50.0} & \textbf{51.2} & \textbf{48.6} \\
			\bottomrule
		\end{tabular}
	\end{minipage}
	\hfill
	\begin{minipage}[c]{0.4\linewidth}
		\captionof{figure}{Regret as a function of time for Gaussian experiment 2 \label{fig:Gaussian}}
		\centering
		\includegraphics[width=6cm, height=3.2cm]{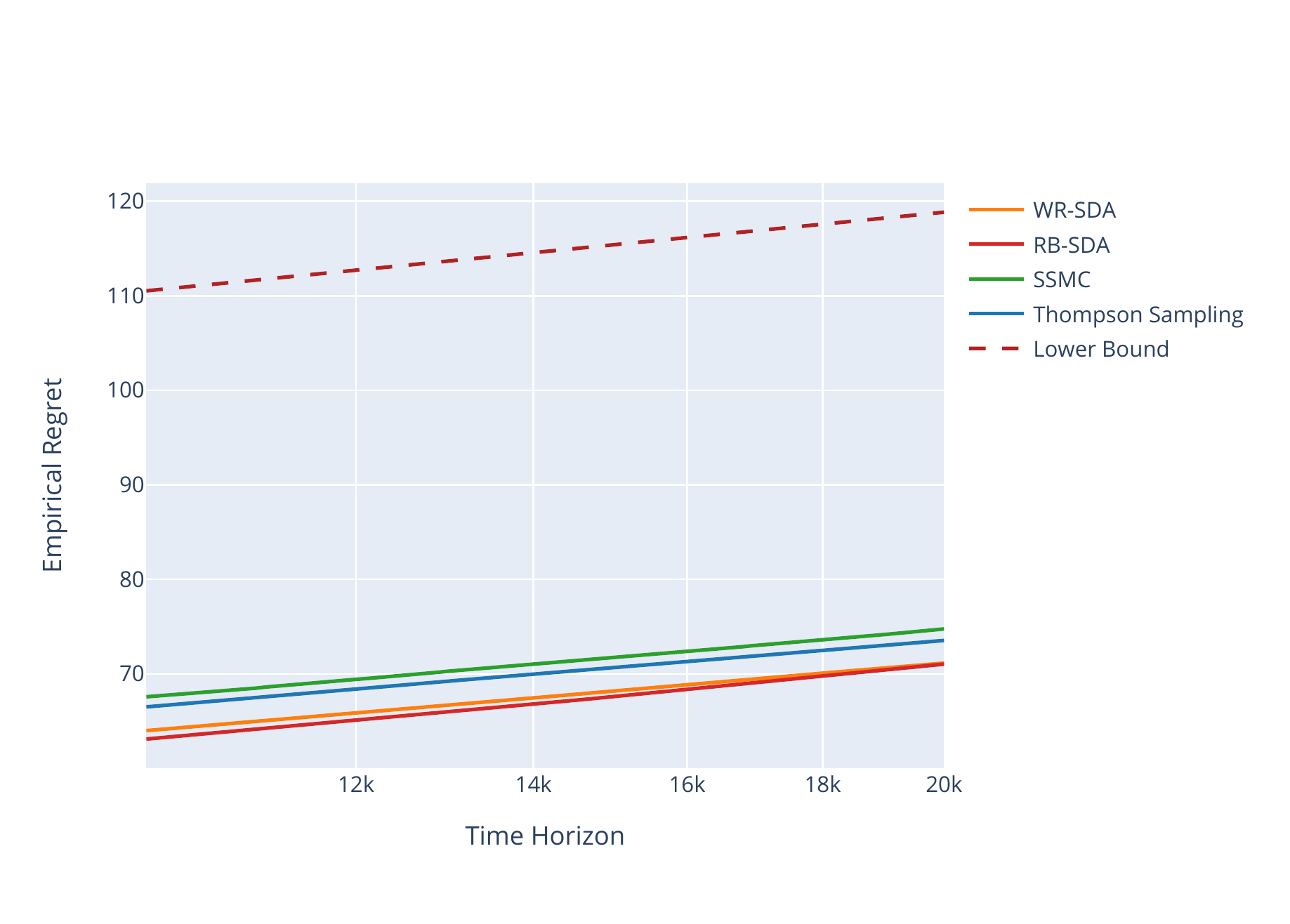}     
	\end{minipage}
	\label{fig:ma_fig}
	
\end{table}

\vspace{-0.2cm}

In all of these experiments, we notice that \SSDA{} algorithms are indeed strong competitors to Thompson Sampling (with appropriate prior) for both Bernoulli and Gaussian bandits. Figures~\ref{fig:Bernoulli} and \ref{fig:Gaussian} further show that \RBSSDA{} is empirically matching the Lai and Robbins' lower bound on two instances, just like SSMC and Thompson Sampling, which can be seen from the parallel straight lines with the $x$ axis in log scale. The fact that the lower bound is above shows that it is really asymptotic and only captures the right first order term. The same observation was made for all experiments, but is not reported due to space limitation. Even if we only established the asymptotic optimality of \RBSSDA{}, these results suggest that the other \SDA{} algorithms considered in this paper may also be asymptotically optimal. Compared to \SDA{}, re-sampling algorithms based on history perturbation seem to be much less robust. Indeed, in the Bernoulli case, PHE performs very well for experiment 2, but is significantly worse than Thompson Sampling on the three other instances. In the Gaussian case, ReBoot always performs significantly worse than other algorithms. This lack of robustness is also corroborated by additional experiments reported below in which we average the performance of these algorithms over a large number of randomly chosen instances. 

Turning our attention to algorithms based on sub-sampling, we first notice that \WRSSDA{} seems to be a better generalization of BESA with $2$ arms than the tournament approach currently proposed, as in experiments with $K>2$, \WRSSDA{} often performs significantly better than BESA. Then we observe that SSMC and SDA algorithms have similar performance. Looking a bit closer, we see that the performance of SSMC is very close to that of \LBSSDA{}, whereas SDA algorithms based on ``randomized'' (or pseudo-randomized for \LDSSSDA{}) samplers tend to perform slightly better. 

\paragraph{Truncated Gaussian} Theorem~\ref{th::log_regret_sda} suggests that \RBSSDA{} may attain logarithmic regret beyond exponential families. As an illustration, we present the results of experiments performed with Truncated Gaussian distributions (in which the distribution of arm $k$ is that of $Y_k = 0 \vee (X_k \wedge 1)$ where $X_k \sim \mathcal{N}(\mu_k,\sigma^2)$). We report in Table~\ref{gaussian-xp} the regret at time $T=20000$ (estimated over $5000$ runs) of various algorithms on four different problem instances: 1) $\mu=[0.5, 0.6]$, $\sigma=0.1$, 2) $\mu=[0, 0.2]$, $\sigma=0.3$, 3) $\mu=[1.5, 2]$, $\sigma=1$ 4) $\mu=[0.4, 0.5, 0.6, 0.7]$, $\sigma=1$. We include Non-Parametric TS which is known to be asymptotically optimal in this setting (while TS which uses a Beta prior and a binarization trick is not), PHE, and all algorithms based on sub-sampling. We again observe the good performance of SSMC and \SDA{} algorithms across all experiments. They even outperform NP-TS in some experiments, which suggests \SDA{} algorithms may be asymptotically optimal for a wider class of parametric distributions.    

\begin{table}[H]
		\caption{Regret at $T=20000$ for Truncated Gaussian arms \label{gaussian-xp}}
		\centering
		\addtolength{\tabcolsep}{-4pt}
		\begin{tabular}{l|l|l|l|l|l|l|l|l|l}
			\toprule
			& \multicolumn{5}{c|}{Benchmark} & \multicolumn{4}{c}{SDA} \\
			\cmidrule(r){2-9}
			xp      &TS& NP-TS & PHE& BESA& SSMC & RB& WR & LB & LDS \\
			\midrule
			1  &21.9& 4.2 & 22.3 & \textbf{1.4}&\textbf{1.5} & \textbf{1.4}& \textbf{1.4}& \textbf{1.5}& \textbf{1.4} \\
		2 & 13.3&8 &19.5 & \textbf{4.6}& \textbf{4.7} & \textbf{4.4} & \textbf{4.5}& \textbf{4.6}& \textbf{4.3} \\
			3 & 9.7 &\textbf{7.8} & 48.5 &\textbf{7.8} & \textbf{7.6} & \textbf{7.1} & \textbf{7.7}& 8.2& \textbf{7.1}\\
			4&  86.6&\textbf{70} & 86& 76.5& \textbf{69.5}& \textbf{64.9}& \textbf{64.8}& \textbf{68.7}& \textbf{63.2}\\
			\bottomrule
		\end{tabular}
\end{table}

\paragraph{Bayesian Experiments}

So far we tried our algorithms on specific instances of the distributions we considered. It is also interesting to  check the robustness of the algorithms when the means of the arms are drawn at random according to some distribution. In this section we consider two examples: Bernoulli bandits where the arms are drawn uniformly at random in $[0,1]$, and Gaussian distributions with the mean parameter of each arm itself drawn from a gaussian distribution $\mu_k \sim \mathcal{N}(0, 1)$. In both cases we draw $10000$ random problems with $K=10$ arms and run the algorithms for a time horizon $T=20000$. We experiment with TS, SSMC, \RBSSDA{} and \WRSSDA{} and also add the IMED algorithm (\cite{IMED}) which is an asymptotically optimal algorithm that uses the knowledge of the distribution. We do not add \LDSSSDA{} and \LBSSDA{} as they are similar to \RBSSDA{} and SSMC, respectively. In the Bernoulli case, we also run the PHE algorithm, which  fails to compete with the other algorithms. This is not in contradiction with the results of \cite{PHE} as in the Bayesian experiments of this paper, arms are drawn uniformly in $[0.25, 0.75]$ instead of $[0,1]$. Actually, we noticed that PHE with parameter $a=1.1$ has some difficulties when several arms are close to $1$.  

\begin{figure}[H]
	\begin{minipage}{0.45\textwidth}
		\begin{table}[H]
			\caption{Average Regret on $10000$ random experiments with Bernoulli Arms}
			\label{exp:bayes_xp_ber}
			\centering
			\small\addtolength{\tabcolsep}{-2pt}
			\begin{tabular}{l|l|l|l|l|l|l}
				\toprule
				T& TS& IMED& PHE& SSMC& RB& WR \\
				\midrule
				100&	13.8&	15.1&	16.7&	16.5&	14.8&	14.3\\
				1000&	27.8&	31.9&	39.5&	34.2&	31.8&	30.9\\
				10000&	45.8&	51.2&	72.3&	55.0&	51.1&	50.6\\
				20000&	52.2&	57.6&	85.6&	61.9&	57.7&	57.3\\
				\bottomrule
			\end{tabular}
		\end{table}
	\end{minipage}
	\hfill
	\begin{minipage}{0.45 \textwidth}
		\begin{table}[H]
			\caption{Average Regret on $10000$ random experiments with Gaussian Arms}
			\label{exp:bayes_xp_gauss}
			\centering
			\small\addtolength{\tabcolsep}{-2pt}
			\begin{tabular}{l|l|l|l|l|l}
				\toprule
				T& TS& IMED& WR& RB& SSMC \\
				\midrule
				100&	41.2&	45.1&	38.3&	38.1&	40.6\\
				1000&	76.4&	82.1&	72.7&	70.4&	76.2\\
				10000&	118.5&	124.0&	115.8&	111.8&	120.1\\
				20000&	132.6&	138.1&	130.2&	125.7&	135.1\\
				\bottomrule
			\end{tabular}
		\end{table}
	\end{minipage}
\end{figure}

Results reported in Tables~\ref{exp:bayes_xp_ber} and ~\ref{exp:bayes_xp_gauss} show that \RBSSDA{} and \WRSSDA{} are strong competitors to TS and IMED for both Bernoulli and Gaussian bandits. Recall that these algorithm operate without the need for a specific tuning for each distribution, unlike TS and IMED. Moreover, observe that in the Bernoulli case, TS further uses the same prior as that from which the means are drawn. 

\paragraph{Computational aspects} To choose a sub-sampling based algorithm, numerical consideration can be taken into account. First, compared to Thompson Sampling, all sub-sampling based algorithm require to store the history of the observation. But then, the cost of sub-sampling varies across algorithms: in the general case \RBSSDA{} is more efficient than \WRSSDA{} as the latter requires to draw a random subset while the former only needs to draw the random integer starting the block. However, for distributions with finite supports \WRSSDA{} can be made as efficient as TS using multivariate geometric distributions, just like PHE does. If one does not want to use randomization then \LDSSSDA{} could be preferred to \RBSSDA{} as it uses a deterministic sequence. Finally, \LBSSDA{} has the smallest computational cost in the \SDA{} family and while its performance is very close to that of SSMC, it can avoid the cost of scanning all the sub-sample means in this algorithm. The computational cost of these two algorithms is difficult to evaluate precisely. Indeed, they can be made very efficient when the leader does not change, but each change of leader is costly, in particular for SSMC. The expected number of such changes is proved to be finite, but for experiments with a finite time horizon the resulting constant can be big. Finally, Non-Parametric TS has a good performance for Truncated Gaussian, but the cost of drawing a random probability vector over a large history is very high.

\paragraph{More experiments} 
In Appendix~\ref{app::complement_xp} we enhance this empirical study: we show some limitations of \SSDA{} for exponential distributions and propose a fix using forced exploration as in SSMC. 

\section{Conclusion}

We introduced the \SSDA{} framework for exploration in bandits models. We proved that one particular instance, \RBSSDA{}, combines both optimal theoretical guarantees and good empirical performance for several distributions, possibly with unbounded support. Moreover, \SSDA{} can be associated with other samplers that seem to achieve similar performance, with their own specificity in terms of computation time. The empirical study presented in the paper also shows the robustness of \textit{sub-sampling} approach over other types of \textit{re-sampling} algorithms.
This new approach to exploration may be generalized in many directions, for example to contextual bandits or reinforcement learning, where UCB and Thompson Sampling are still the dominant approaches. It is also particularly promising to develop new algorithm for non-stationary bandit, as such algorithms already store the full history of rewards. 

\newpage

\begin{ack}
The PhD of Dorian Baudry is funded by a CNRS80 grant. The authors acknowledge the funding of the French National Research Agency under projects BADASS (ANR-16-CE40-0002) and BOLD (ANR-19-CE23-0026-04).

Experiments presented in this paper were carried out using the Grid'5000 testbed, supported by a scientific interest group hosted by Inria and including CNRS, RENATER and several Universities as well as other organizations (see \hyperlink{https://www.grid5000.fr}{https://www.grid5000.fr}).
\end{ack}

\bibliographystyle{unsrt}
\bibliography{biblio}

\newpage
\appendix

\input{appendix_A}

\input{appendix_B}

\input{appendix_C}
\input{appendix_D}
\input{appendix_D2}
\input{appendix_E}

\input{appendix_F}

\input{handling_forced_explo}

\end{document}

%% file: appendix_A.tex
\section{Complement of Experiments}\label{app::complement_xp}

\subsection{Additional Figures for Bernoulli, Gaussian and Truncated Gaussian arms}\label{app::add_xp_BGTG}

We enhance the tables of Section~\ref{sec:expes} with the standard deviation (reported in parenthesis) of the regret at time $T=20000$ on the 5000 trajectories.

\begin{table}[H]
		\caption{Regret and at $T=20000$ for Bernoulli arms, with standard deviation}
		\label{bernoulli-xp-big}
		\centering
		\addtolength{\tabcolsep}{-4pt}
		\begin{tabular}{l|l|l|l|l|l|l|l|l}
			\toprule
			& \multicolumn{4}{c|}{Benchmark} & \multicolumn{4}{c}{SSDA} \\
			\cmidrule(r){2-9}
			xp      & TS & PHE & BESA & SSMC & RB& WR & LB & LDS \\
			\midrule
			1      & \textbf{11.2}  & 25.9 & \textbf{11.7}& \textbf{12.3}& \textbf{11.5} & \textbf{11.6} & \textbf{12.2}& \textbf{11.4} \\
			& (10.)& (87.9)& (12.1)& (7.3)&(10.1)&(10.2)&(7.4)& (9.0) \\
			\midrule
			2      & \textbf{22.9} & \textbf{24.0} &\textbf{22.1} & \textbf{24.3} & \textbf{22.0} & \textbf{21.5} & \textbf{24.0} & \textbf{21.8} \\
			&(29.2)&(22.0)&(25.2)&(38.2)&(34.5)&(17.3)&(24.6)&(24.5) \\
			\midrule
			3      & \textbf{94.2} &  248.1 & \textbf{88.1} & 100.1 & \textbf{89.0} & \textbf{86.9} & 100.7 & \textbf{89.2} \\
			& (15.8)&	(25.5)&	(89.2)&	(20.0)&	(19.8)&	(21.7)&	(21.3)&	(21.8)\\
			\midrule
			4      & \textbf{108.1} & 216.5 & 147.5 & 119.9 & \textbf{105.1} & \textbf{106.9} & 119.6 & \textbf{106.8}\\
			& (45.1)&(89.8)	&(209.8)&(40.8)&(41.1)&(42.1)&(42.7)&(47.7)\\
			\bottomrule
		\end{tabular}
\end{table}

\begin{table}[H]
	\caption{Regret and at $T=20000$ for Gaussian arms, with standard deviation}
		\label{gaussian-xp}
		\centering
	 \addtolength{\tabcolsep}{-4pt}
		\begin{tabular}{l|l|l|l|l|l|l|l|l}
			\toprule
			& \multicolumn{4}{c|}{Benchmark} & \multicolumn{4}{c}{SDA} \\
			\cmidrule(r){2-9}
			xp      & TS     & ReBoot & BESA & SSMC & RB& WR & LB & LDS \\
			\midrule
			1      & \textbf{24.4} &  92.2& \textbf{25.3} & \textbf{26.9} & \textbf{25.6} & \textbf{24.7} & \textbf{25.1} & \textbf{26.5} \\
			&(17.1)&(23.4)&	(27.1)&	(52.8)&	(62.8)&	(20.6)&	(17.9)&	(140.2)\\
			\midrule
			2& \textbf{73.5} & 277.1&122.5 & \textbf{74.8} & \textbf{71.0} & \textbf{71.1} & \textbf{74.6} & \textbf{69.0} \\
			& (107.8)&	(41.3)&	(585.5)& (34.7)& (152.2)& (50.2)&(35.1)& (50.4)\\
			\midrule
			3 & \textbf{49.7} & 190.9&72.1 & \textbf{51.3} & \textbf{50.4} & \textbf{50.0} & \textbf{51.2} & \textbf{48.6} \\
			& (26.9)&	(29.6)&	(410.3)&(23.7)&	(156.5)&(33.3)&	(22.4)&	(41.6)\\
			\bottomrule
		\end{tabular}	
\end{table}

\begin{table}[H]
	\caption{Regret at $T=20000$ for Truncated Gaussian arms \label{gaussian-xp}}
	\centering
	\addtolength{\tabcolsep}{-4pt}
	\begin{tabular}{l|l|l|l|l|l|l|l|l|l}
		\toprule
		& \multicolumn{5}{c|}{Benchmark} & \multicolumn{4}{c}{SDA} \\
		\cmidrule(r){2-9}
		xp      &TS& NP-TS & PHE& BESA& SSMC & RB& WR & LB & LDS \\
		\midrule
		1  &21.9& 4.2 & 22.3 & \textbf{1.4}&\textbf{1.5} & \textbf{1.4}& \textbf{1.4}& \textbf{1.5}& \textbf{1.4} \\
					& (20.4)&(0.6)& (2.6)&(1.7)&(0.7)&(1.1)&(0.8)&(0.7)&(0.8) \\
					\midrule
		2 & 13.3&8 &19.5 & \textbf{4.6}& \textbf{4.7} & \textbf{4.4} & \textbf{4.5}& \textbf{4.6}& \textbf{4.3} \\
					& (7)&(1.8)& (3.8)&(3.3)&(2.3)&(4.6)&(3.1)&(2.4)&(2.9) \\			\midrule
		3 & 9.7 &\textbf{7.8} & 48.5 &\textbf{7.8} & \textbf{7.6} & \textbf{7.1} & \textbf{7.7}& 8.2& \textbf{7.1}\\
		& (10.1)&(4.5)&(217.8)&(9.4)&(5)&(10)&(13.4)&(27.5)&(5.8)\\
					\midrule
		4&  86.6&\textbf{70} & 86& 76.5& \textbf{69.5}& \textbf{64.9}& \textbf{64.8}& \textbf{68.7}& \textbf{63.2}\\
		& (57.8)& (39.4)&(53.7)&(113.9)&(40.9)&(60.5)&(43.9)&(39.1)&(51.1)\\
		\bottomrule
	\end{tabular}
\end{table}

For Bernoulli arms and Truncated Gaussian, the standard deviations of \SSDA{} are very similar to that of Thompson Sampling, while the trajectories of PHE and BESA have much more variance in experiment 1 and 4, and on experiments 3 and 4 respectively. For Gaussian arms we remark the low variability of ReBoot, but at the cost of a non-competitive regret. \SSDA{} are less homogeneous in this case: some algorithms have large variance for some instances (\LDSSSDA{} on experiment 1, \RBSSDA{} on experiments 2 and 3). Note that TS also has a high variability in experiment 2. 

We believe that this is due to the nature of the Gaussian distribution, and in particular to its balance function:  in Appendix~\ref{app::balance} we prove that $\alpha_k(M,j)$ does satisfy Assumption 2. of Theorem~\ref{th::log_regret_sda}, however the upper bound derived for $\alpha_k(M,j)$ is much larger than the one for Bernoulli distribution, which justifies that ``bad runs'' in which a good arm looses many duels are more likely to happen in that case, and can explain the larger variance. If one wants to reduce the variance of the regret of \SSDA{} we recommend the use of some asymptotically negligible forced exploration, as presented for exponential distribution in Appendix~\ref{app::xp_exp_arms}, and for which we prove that the algorithm remains asymptotically optimal in Appendix~\ref{app::forced_explo}. 

Finally, as in Section~\ref{sec:expes}, we plot the regret of several algorithms as a function of time (in log scale) for $t \in [10000, 20000]$, this time for the Truncated Gaussian distributions. These plots illustrate the fact that some \SSDA{} algorithms may achieve asymptotic optimality for this distribution too, even if it does not belong to a one-parameter exponential family. Indeed, the rate of the regret of all \SSDA{} seem too match both the rate of the regret of Non-Parametric TS, which is optimal for this family, and the Burnetas and Katehakis lower bound whose expression is $\left(\sum_{k \neq k^\star} \frac{\bE_{X\sim \nu_{k^*}}[X] - \bE_{X\sim \nu_{k}}[X]}{\mathrm{KL}(\nu_k,\nu_{k^*})}\right)\log(T)$ in this particular case, with $\mathrm{KL}(\nu_k,\nu_{k^*}) = p_{0, k} \log\left(\frac{p_{0, k}}{p_{0, *}}\right)+(1-p_{1, k}) \log\left(\frac{1-p_{1, k}}{1-p_{1, *}}\right) + \int_{0}^1 f_k(x) \log\left(\frac{f_k(x)}{f_*(x)}\right) dx$. $p_{x, k}$ is the value of the CDF of the underlying Gaussian random variable associated with $\nu_k$ in $x$, and $f_k(x)$ the density of this variable in $x$.

\begin{figure}[H]
	\begin{minipage}[c]{0.5\linewidth}
	\caption{\SSDA{} vs NP-TS on TG expe 2}
	\centering
	\includegraphics[scale=0.4]{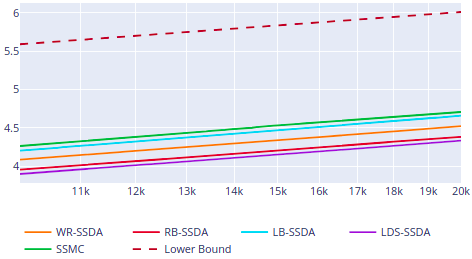}
\end{minipage}
\begin{minipage}[c]{0.5\linewidth}
	\caption{\SSDA{} vs NP-TS on TG expe 4}
	\centering\includegraphics[scale=0.35]{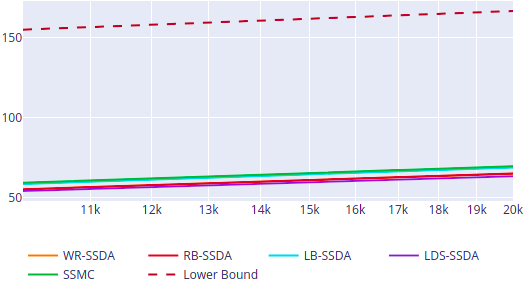}
\end{minipage}
\end{figure}	

\subsection{Experiments with Exponential Arms}\label{app::xp_exp_arms}

In Appendix~\ref{app::balance}, we prove that exponential distributions are \emph{not} balanced (i.e. do not satisfy Assumption 2. of Theorem~\ref{th::log_regret_sda}), so our theoretical results on the regret of \RBSSDA{} do not apply. However, it is still interesting to test our algorithms for these distributions in order to see if it still achieves a good performance. We performed 6 experiments, with the following mean parameters: 1) $\mu=[1.5, 1]$, 2) $\mu=[0.2, 0.1]$, 3) $\mu=[11, 10]$, 4) $\mu=[4, 3, 2, 1]$, 5) $\mu=[0.4, 0.3, 0.2, 0.1]$, 6) $\mu=[5,4,4,4]$. It is interesting to remark that the standard deviation of an exponential distribution is equal to its mean, so with similar gaps problems are harder when the means are high.

\begin{table}[H]
	\caption{Average Regret with Exponential Arms (with std)}
	\label{exp-xp}
	\centering
	\small\addtolength{\tabcolsep}{-2pt}
	\begin{tabular}{l|l|l|l|l|l|l|l|l}
		\toprule
		xp      & TS & IMED & BESA& SSMC & RB& WR & LB & LDS \\
		\midrule
		1 & 48.2 & \textbf{40.0}& 45.7& \textbf{41.9}& 44.8& 45.4& 46.6& 45.5\\
		&(191.8)&(78.4)&(114.1)	&(84.2)&(121.4)&(134.4)&(176.8)&(109.7)\\
		\midrule
		2 & 3.8&  \textbf{3.4}& 4.2& \textbf{3.6}& 4.1& 3.9& 3.9& 5.4\\
		& (9.9)&(3.6)&(25.1)&(41.9)&(14.3)&(13.4)&(8.7)&(49.5)\\
		\midrule
		3 & 832.8& \textbf{779.9}& 820.5& 856.9& 848.4& \textbf{778.4}& 846.7& 877.7\\
		& (1065.1)&(896.9)&(1304.6)&(1111.0)&(1533.3)&(1118.7)&(1150.1)&(1708.7)\\
		\midrule
		4& 258.3& \textbf{234.6}& 525.4& \textbf{251.3}& 272.6& 262.1& 263.8& 258.4\\
		&(519.6)&(126.6)&(2115.1)&(328.3)&(692.2)&(524.4)&(477.9)&(599.0)\\
		\midrule
		5& \textbf{25.6}& \textbf{24.0}& 55.7& \textbf{25.6}& \textbf{25.5}& \textbf{25.0}& 26.5& \textbf{24.7}\\
		&(51.2)&(33.6)&(219.9)&(23.6)&(46.7)&(24.0)&(36.8)&(37.6)\\
		\midrule
		6& \textbf{618.7}& \textbf{603.6}& 1184.2& \textbf{627.9}& \textbf{595.7}& \textbf{616.0}& 652.6& \textbf{605.9}\\
		& (672.3)&(576.8)&(3096.4)&(755.6)&(790.7)&(780.2)&(685.3)&(871.4)\\
		\bottomrule
	\end{tabular}
\end{table}

First, we notice that the performance of the \SDA{} in terms of the average regret are reasonable, although less impressive than with the other distributions we tested. IMED is almost always the best algorithm in these experiments, and  SSMC performs pretty well on many examples (which is not surprising as SSMC is proved to be asymptotically optimal for exponential distributions). 
We remark that there is much more variability in the results of \RBSSDA{}, \WRSSDA{} and \LDSSSDA{} than before, where they performed quite similarly. For instance, we notice that on example 3, \LDSSSDA{} and \RBSSDA{} are much worse than \WRSSDA{}. A look at the quantile table for this experiment, which displays the empirical quantiles of $R_T$ estimated over $5000$ runs, shows that this is due to a small number of "bad" trajectories for these algorithms: 

\begin{table}[H]
	\caption{Quantile Table for Experiment 3 with Exponential Arms}
	\label{exp-xp}
	\centering
	\small\addtolength{\tabcolsep}{-2pt}
	\begin{tabular}{c|l|l|l|l|l|l|l}
		\toprule
		\% of runs      & TS & IMED & SSMC & RB& WR & LB & LDS \\
		\midrule
		20\% & 319.8& 336.0&	335.0	&261.0&	290.0&	326.0&	261.8\\
		50\%& 626.0&	650.0&	661.0	&532.0&	568.5	&642.0&	536.0 \\
		80\%& 1122.0&	1080.0	&1142.0&	1006.0&	1019.0&	1143.2&	1020.2\\
		95\%& 1924.1&	1704.0	&1846.0&	\textbf{2199.0}	&1817.2&	1869.1&	\textbf{2134.1}\\
		99\%& 4209.4&	2632.9&	3536.8&	\textbf{6813.1}&	4146.0&	3762.3&	\textbf{7396.7}\\
		\bottomrule
	\end{tabular}
\end{table}

We see that up to the $80\%$ quantile, \RBSSDA{} and \LDSSSDA{} are even significantly better than IMED. This is very different when we look at the $95\%$ and $99\%$ quantiles, which are much greater for our 2 algorithms (even $2.5$ times greater for the $99\%$ quantile).

We believe that this very high variability prevents \RBSSDA{} to have a logarithmic regret for exponential arms. Still, the regret is not as bad as being linear, as using the fact that the balance function $\alpha_k(M,j)$ is of order ${\exp(-j C)}/{M}$ permits to prove that $\sum_{r=1}^{T} \bP(N_1(r) < \log^2(r)) = \mathcal{O}(\log^2(T))$ (which requires to choose a different $\beta_{r,j}$ in Lemma~\ref{lemma::decomposition}). But we also found a solution to obtain (asymptotically optimal) linear regret, which consists in adding an asymptotically negligible amount of \textit{forced exploration} as the SSMC algorithm does. This exploration in $o(\log T)$ avoids trajectories where the optimal arm has a very bad first observation and is not drawn for a very long time. In Appendix~\ref{app::forced_explo}, we prove the asymptotic optimality of \RBSSDA{} with forced exploration $f_r = \sqrt{\log r}$ for any one-dimensional exponential family. In practice, adding this amount of forced exploration to \SDA{} algorithms leads to the following results:

\begin{table}[H]
	\caption{Average Regret with Exponential Arms: \SDA{} with forced exploration}
	\label{exp-xp-explo}
	\centering
	\small\addtolength{\tabcolsep}{-2pt}
	\begin{tabular}{l|l|l|l|l}
		\toprule
		xp      & RB& WR & LB & LDS \\
		\midrule
		1& 44.9& \textbf{42.5}& \textbf{42.4}& 45.0\\
		&(167.3)&(107.4)&(60.5)&(176.0)\\
		\midrule
		2& \textbf{3.6}& \textbf{3.4}& 4.0& \textbf{3.6}\\
		&(9.2)&(2.2)&(27.9)&(11.2)\\
		\midrule
		3& 837.5& \textbf{788.5}& 827.7& 832.3\\
		& (1466.1)&(1222.1)&(1055.3)&(1514.6)\\
		\midrule
		4& \textbf{244.8}& \textbf{238.9}& 251.7& \textbf{246.0}\\
		&(403.3)&(250.8)&(248.5)&(323.4)\\
		\midrule
		5& \textbf{23.6}& \textbf{25.1}& \textbf{25.4}& \textbf{24.9}\\
		&(33.4)&(41.0)&(23.4)&(42.2)\\
		\midrule
		6& \textbf{578.9} & \textbf{595.1}& 631.2& \textbf{577.8}\\
		& (651.9)&	(561.3)&(446.4)&(652.7)\\
		\bottomrule
	\end{tabular}
\end{table}

Hence, adding forced exploration results in a noticeable improvement for \SDA{} algorithms, with \RBSSDA{}, \WRSSDA{} and \LDSSSDA{} becoming competitive with IMED  (or even slightly better) on most examples. Observe that \LBSSDA{} has again comparable performance with SSMC with this new feature. This is not surprising as we implemented the SSMC algorithm with the same amount of forced exploration $f_r = \sqrt{\log r}$ .

%% file: appendix_B.tex
\newpage \section{Notation for the Proof}\label{app::notations}

General notations:
\begin{itemize}
	\item $K$ number of arms
	\item $\nu_k$ distribution of the arm $k$, with mean $\mu_k$
	\item we assume that $\mu_1 = \max_{k \in [K]} \mu_k$ so we call the (unique) optimal arm "arm 1"
	\item $I_k(x)$ some rate function of the arm $k$, evaluated in $x$. For 1-parameter exponential families this function will always be the KL-divergence between $\nu_k$ and the distribution from the same family with mean $x$.
	\item $N_k(r)$ number of pull of arm $k$ up to (and including) round $r$.
	\item $Y_{k, i}$ reward obtained at the i-th pull of arm $k$.
	\item $\hat Y_{k, i}$ mean of the i-th first reward of arm $k$, $\hat Y_{k, \cS}$ mean of the rewards of $k$ on a subset of indices $\cS \subset [N_k(r)] $:  $\hat Y_{k, \cS}= \frac{1}{|\cS|} \sum_{s \in \cS} Y_{k, s}$. If $|\cS|=i$, then $Y_{k, i}$ and $Y_{k, \cS}$ have the same distribution.
	\item $\ell(r)$ leader at round $r+1$, $\ell(r)=\text{argmax}_{k \in [K]} N_k(r)$.
	\item $\text{SP}(m, n, r)$ sub-sampling algorithm, or Sampler, which returns a sequence of $n$ unique elements out of $[m]$.
	\item $(S_k^r(m, n))_{m \geq n}$ a family of independent random variables such that $S_k^r(m, n) \sim \text{SP}(m, n, r)$. 
	\item $\cA_r$ set of arms pulled at a round $r$.
	\item $\mathcal{R}_r$ regret at the \textit{end} of round $r$.
\end{itemize}

Notations for the regret analysis, part relying on concentration:
\begin{itemize}
	\item $\cG_k^r = \cup_{s=1}^{r-1} \{\ell(s)=1\} \cap  \{k \in \mathcal{A}_{s+1}\} \cap \{N_k(s) \geq (1+\epsilon) \xi_k \log r\}$
	\item $\cH_k^r = \cup_{s=1}^{r-1} \{\ell(s)=1\} \cap  \{k \in \mathcal{A}_{s+1}\} \cap \{N_k(s) \geq J_k \log r\}$ 
	
	\item $\cZ^r = \{\ell(r) \neq 1\}$, the leader used for the duels in round  $r+1$ is sub-optimal
	\item $\cD^r = \{\exists u \in  [\lfloor r/4 \rfloor, r] \text{ such that } \ell(u-1) = 1 \}$, the leader has been optimal at least once between $\lfloor r/4 \rfloor$ and $r$
	\item $\cB^{u} = \{\ell(u)=1, k \in \mathcal{A}_{u+1}, N_k(u)=N_1(u)-1 \text{ for some arm } k \}$, the optimal arm is leader in $u$ but loses its duel again some arm $k$, that have been pulled enough to possibly take over the leadership at next round
	\item $\cC^u = \{\exists k \neq 1, N_k(u)\geq N_1(u), \hat Y_{k,S_1^u(N_k(u),N_{1}(u))} \geq \hat Y_{1,N_1(u)}\}$, the optimal arm is not the leader and has lost its duel against the sub-optimal leader.
	\item $\cL^r= \sum_{u=\lfloor r/4 \rfloor}^r \ind_{\cC^u}$
\end{itemize}

Notations for the regret analysis, control of the number of pulls of the optimal arm:
\begin{itemize}
	\item $r_j$ round of the j-th play of the optimal arm
	\item $\tau_j = r_{j+1}-r_j$
	\item $\cE_j^r := \{\tau_j \geq r/\log r^2 -1\}$
	\item $\cM_{j, r}^1 = \left[r_j+1, r_j+ \left\lfloor \frac{r/\log r^2 -1}{2}\right\rfloor\right]$
	\item $\cM_{j, r}^2 = \left[t_j+ \left\lceil \frac{r/\log r^2 -1}{2}\right\rceil, r_j+\left\lfloor r/\log r^2 \right\rfloor -1\right]$
	\item $\cI_{j, r}^k = \{s \in \cM_{j, r}^2 : \ell(s-1)=k \}$
	\item $\cW_{s,j}^{k} = \left\{\left\{ \hat Y_{1, j} < \hat Y_{k,S_1^s(N_k(s), j)} \right\}, N_k(s)\geq c_{r,K}, N_1(s)=j\right\}$
	\item $\cF_{j, M}^{k, r} = \left\{\exists i_1,..., i_M \in I_{j, r}^k: \forall m<m' \in [M], S_{1}^{i_m}(N_{k}(i_m),j) \cap S_{1}^{i_{m'}}(N_{k}(i_{m'}),j) = \emptyset \right\}$
	\item CDF: Cumulative Distribution Function, PDF: Probability Density Function and PMF: Probability Mass Function.
\end{itemize}

%% file: appendix_C.tex
\section{Concentration Result: Proof of Lemma~\ref{concentration}}\label{app::lemma_concentration}

We first recall the probabilistic model introduced in Section~\ref{sec:SSDA}: for each round $r$, each arm $k$, we define a family $(S_k^r(n,m))_{n > m}$ of independent random variables such that $S_k^r(n,m) \sim \text{SP}(n,m, r)$. Those random variables are also independent from the reward streams $(Y_{k,s})_{s \geq 0}$ of all arms $k$. 

$S_k^r(n,m)$ is the subset of the leader history that is used should arm $k$ be a challenger drawn $m$ times up to round $r$ duelling against a leader that has been drawn $n$ times. With this notation, letting $\ell(r)$ be the leader after $r$ rounds, at round $r+1$, for all $k\neq \ell(r)$,
\[(k \in \cA_{r+1}) \ \Leftrightarrow \ \left(\hat{Y}_{k,N_k(r)} > \hat{Y}_{\ell(r),S^r_k\left(N_{\ell(r)}(r),N_k(r)\right)}\right)\;.\]
 
Let $k$ be an arm such that $\mu_k < \mu_1$. We denote by $[n_1,n_k]$ the set of subset of $\{1,\dots,n_1\}$ of size $n_k$. We define an event \[\mathcal{Q}_k^s = \{N_k(s)\geq n_0, \ell(s)=1, \hat{Y}_{k,N_k(s)} > \hat{Y}_{\ell(s),S^{s}_k\left(N_1(s),N_k(s)\right)}\}.\] Noting that $\{\hat{Y}_{k,N_k(s)} > \hat{Y}_{\ell(s),S^{s}_k\left(N_1(s),N_k(s)\right)}\} \subset \{\hat{Y}_{k,N_k(s)} \geq \xi\} \cup \{ \hat{Y}_{\ell(s),S^{s}_k\left(N_1(s),N_k(s)\right)}\leq \xi\}$ for all $\xi \in \R$, we can write $\mathcal{Q}_k^s \subset \mathcal{Q}_{k}^{s, 1} \cup \mathcal{Q}_{k}^{s, 2}$ where 
\begin{eqnarray*}\mathcal{Q}_{k}^{s, 1} &= &\{N_k(s)\geq n_0, \ell(s)=1, \hat{Y}_{k,N_k(s)} > \xi\} \\ \text{and } \ \mathcal{Q}_{k}^{s, 2} &=&\{N_k(s)\geq n_0, \ell(s)=1, \hat{Y}_{\ell(s),S^{s}_k\left(N_1(s),N_k(s)\right)} \leq \xi \}.\end{eqnarray*} 
This yields $\sum_{s=1}^{r}\bP(\mathcal{Q}_k^s) \leq \sum_{s=1}^{r}\bP(\mathcal{Q}_k^{r,1})+ \sum_{s=1}^{r}\bP(\mathcal{Q}_k^{r,2})$, which will later provide the two terms in the bound of the lemma. The first one does not involve sub-sampling and can be upper bounded as:
\begin{align*}
& \sum_{s=1}^{r}\bP(\mathcal{Q}_k^{s,1}) \leq \bE \sum_{s = 1}^{r} \ind(N_k(s)\geq n_0) \ind(N_1(s)>N_k(s)) \ind\left(\hat Y_{k,N_k(s)} \geq \xi \right)\ind\left(k \in \cA_{s+1}\right)  \\
& \leq \bE \sum_{s = n_0}^{r} \sum_{n_k =n_0}^r \ind\left(N_k(s)=n_k, k \in \cA_{s+1}\right) \ind\left(\hat Y_{k,n_k} \geq \xi \right)\\
& \leq \bE \sum_{n_k =n_0}^r \ind\left(\hat Y_{k,n_k} \geq \xi \right)  \underbrace{\sum_{s = n_0}^{r} \ind\left(N_k(s)=n_k, k \in \cA_{s+1}\right)}_{\leq 1}\\
& \leq \sum_{n_k =n_0}^r \bP\left(\hat Y_{k,n_k} \geq \xi \right),
\end{align*}
where in the last inequality we use that the event $(N_k(s) = n) \cap (k \in \cA_{s+1})$ can happen at most once for $s \in \{n_0,\dots,r\}$ (a similar trick was used for example in the analysis of $\kl$-UCB \cite{AOKLUCB}). 

Upper bounding the second term $B_r = \sum_{s=1}^{r}\bP(\mathcal{Q}_k^{r,2})$ is more intricate as it involves both $N_k(s)$ and $N_1(s)$. With a similar method we get:
\begin{align*}
& B_r\leq \bE \sum_{s = n_0}^{r} \sum_{n_k = n_0}^r\sum_{n_1 =n_k}^r \sum_{\cS \in [n_1,n_k]} \ind\left(N_k(s) = n_k, k \in \cA_{s+1}\right)\ind\left(N_1(s) = n_1\right)\ind\left(S_k^s(n_1,n_k) = \cS\right)\ind\left(\hat Y_{\ell, \cS} \leq \xi\right) \\
& \leq \bE  \sum_{n_k = n_0}^r\sum_{n_1 = n_k}^r \sum_{\cS \in [n_1,n_k]}\ind\left(\hat Y_{\ell, \cS}\leq \xi\right) \sum_{s = n_0}^{r}\ind\left(N_k(s) = n_k, k \in \cA_{s+1}\right)\ind\left(S_k^s(n_1,n_k) = \cS\right)\\
& = \bE  \sum_{n_k = n_0}^r\sum_{n_1 = n_k}^r \sum_{\cS \in [n_1,n_k]}\ind\left(\hat Y_{\ell, \cS}\leq \xi\right) \sum_{s = n_0}^{r}\bE\left[\ind\left(N_k(s) = n_k, k \in \cA_{s+1}\right)\ind\left(S_k^s(n_1,n_k) = \cS\right) | \cF\right],\\
\end{align*}
where $\cF$ is the filtration generated by the reward streams. $N_k(s)$ may have a complicated distribution with respect to this filtration but this is not a problem here. Indeed, $S_k^s(n_1,n_k)$ is by design independent of this filtration, and one can write 

\begin{align*}
& B_r \leq \bE  \sum_{n_k = n_0}^r\sum_{n_1 = n_k}^r \sum_{\cS \in [n_1,n_k]}\ind\left(\hat Y_{1, \cS}\leq \xi\right) \sum_{s = n_0}^{r} \bP\left(S_k^s(n_1,n_k) = \cS\right)\bE\left[\ind\left(N_k(s) = n_k, k \in \cA_{s+1}\right) | \cF\right]\\
& = \bE  \sum_{n_k = n_0}^r\sum_{n_1 = n_k}^r \sum_{\cS \in [n_1,n_k]}\ind\left(\hat Y_{1, \cS}\leq \xi \right) \sum_{s = n_0}^{r} \bP\left(S_k^s(n_1,n_k) = \cS\right)\ind\left(N_k(s) = n_k, k \in \cA_{s+1}\right) \\
& = \sum_{n_k = n_0}^r\sum_{n_1 = n_k}^r \sum_{\cS \in [n_1,n_k]}\bP\left(\hat Y_{1, \cS}\leq \xi \right) \sum_{s = n_0}^{r} \bP\left(S_k^s(n_1,n_k) = \cS\right) \bE \left(\ind\left(N_k(s) = n_k, k \in \cA_{s+1}\right) \right) \\
& = \sum_{n_k = n_0}^r\sum_{n_1 = n_k}^r \bP\left(\hat Y_{1, n_1}\leq \xi \right) \sum_{s = n_0}^{r} \left(\sum_{\cS \in [n_1,n_k]}\bP\left(S_k^s(n_1,n_k) = \cS\right)\right) \bE \left(\ind\left(N_k(s) = n_k, k \in \cA_{s+1}\right) \right)\\
& = \sum_{n_k = n_0}^r\sum_{n_1 = n_k}^r \bP\left(\hat Y_{1, n_1}\leq \xi \right) \bE \underbrace{\sum_{s = n_0}^{r}\ind\left(N_k(s) = n_k, k \in \cA_{s+1} \right)}_{\leq 1}\\
& \leq \sum_{n_k = n_0}^r\sum_{n_1 = n_k}^r \bP\left(\hat Y_{1, n_1}\leq \xi \right)\\
& \leq r \sum_{n_1 = n_k}^r \bP\left(\hat Y_{1, n_1}\leq \xi \right)\;.
\end{align*}
Here we have used the independence of the $S_k^s(m,n)$ from the reward streams and the fact that for every subset $\cS$ of size $n_1$, $\hat{Y}_{k,n_1}$ and $\hat{Y}_{k,\cS}$ have the same distribution. We can conclude as follows, proving the lemma: 

\begin{align*}
 \sum_{s=1}^{r}\bP(\mathcal{Q}_k^r) &\leq \sum_{s=1}^{r}\bP(\mathcal{Q}_k^{r, 1}) + \sum_{s=1}^{r}\bP(\mathcal{Q}_k^{r, 2}) \\ & \leq \sum_{n_k=n_0}^r \bP\left(\hat Y_{k, n_k}\geq \xi \right) + r \sum_{n_1=n_0}^r \bP\left(\hat Y_{1, n_1}\leq \xi \right)\;.
\end{align*}

%% file: appendix_D.tex
\section{Regret Decomposition: Proof of Lemma~\ref{lem:dec}}\label{app::lem_dec}

We recall that we assume that arm $1$ is the only optimal arm: $\mu_1=\max_{k \in [K]} \mu_k$. The proof in this section follows the path of the proof in \cite{SSMC} for SSMC, but hinges on the new concentration result of Lemma~\ref{concentration}. Moreover, some parts need to be adapted to handle the properties of an \sampler{} instead of the duelling rule used in SSMC. As in \cite{SSMC}, we introduce the following events: 
\begin{itemize}
	\item $\cG_k^T = \cup_{r=1}^{T-1} \{\ell(r)=1\} \cap  \{k \in \mathcal{A}_{r+1}\} \cap \{N_k(r) \geq (1+\epsilon) \xi_k \log T\}$ 
	
	\item $\cH_k^T = \cup_{r=1}^{T-1} \{\ell(r)=1\} \cap  \{k \in \mathcal{A}_{r+1}\} \cap \{N_k(r) \geq J_k \log T\}$ 
	
	\item $\cZ^r = \{\ell(r) \neq 1\}$, the leader used at round  $r+1$ is sub-optimal.
\end{itemize}

These events directly provide an upper bound of the number of pulls of a sub-optimal arm $k$:
\begin{align}
\mathbb{E}[N_k(T)]& = \mathbb{E}[N_k(T) \ind_{\cH_k^T}] + \mathbb{E}[N_k(T) \ind_{\cG_k^T} \ind_{\bar{\cH}_k^T}] +\mathbb{E}[N_k(T) \ind_{\bar{\cG}_k^T}]   \nonumber\\
& \leq T \mathbb{P}(\cH_k^T) + (1+J_k \log T)  \mathbb{P}(\cG_k^T) + 1 + (1+\epsilon)\xi_k \log T  + 2 \sum_{r=1}^{T-1} \mathbb{P}(\cZ^r)\label{eq:dec}
\end{align}

Indeed, due to the definition of each event we have:
\begin{align*} N_k(T) \ind_{\bar{\cG}_k^T} 
	& \leq  1 + \sum_{r=1}^{T-1}\ind_{(k \in \cA_{r+1})}\ind_{(\ell(r) \neq 1)\cup(N_{k}(r) < (1+\epsilon)\xi_k \log(T))}\\
	& \leq  1 + \sum_{r=1}^{T-1}\ind_{(k \in \cA_{r+1})}\ind_{(N_{k}(r) < (1+\epsilon)\xi_k \log(T))} + \sum_{r=1}^{T-1}\ind_{(\ell(r) \neq 1)}\\
	& \leq  
	1 + (1+\epsilon) \xi_k \log T + \sum_{r=1}^{T-1} \ind_{\cZ^r}\end{align*}
and similarly
\begin{align*}
N_k(T) \ind_{\cG_k^T} \ind_{\bar{\cH}_k^T} &\leq \left(1 + J_k \log T + \sum_{r=1}^{T-1} \ind_{\cZ^r}\right) \ind_{\cG_k^T} \\
& \leq (1 + J_k \log T) \ind_{\cG_k^T}  + \sum_{r=1}^{T-1} \ind_{\cZ^r}
\end{align*}

Choosing $\xi_k=1/I_1(\mu_k)$ the bound in~\eqref{eq:dec} exhibits the term in $\frac{1+\epsilon}{I_1(\mu_k)} \log T$ in Lemma~\ref{lem:dec}. To obtain the result, it remains to upper bound
\[ T \mathbb{P}(\cH_k^T) + (1+J_k \log T)  \mathbb{P}(\cG_k^T) +  2 \sum_{r=1}^{T-1} \mathbb{P}(\cZ^r)\]
for an appropriate choice of $J_k$. To do so, we shall first use the concentration inequality Lemma~\ref{concentration} to upper bound the terms involving $\cG_k^T$ and $\cH_k^T$ by problem-dependent constants (Appendix~\ref{app::bound_G_H}), and then we carefully handle the terms in $\cZ^r$ (Appendix~\ref{app::boundZ}).

\subsection{Upper Bounds on $\bP (\cG_k^T)$ and $\bP (\cH_k^T)$}\label{app::bound_G_H}

We first fix some real numbers $J_k$ and $\omega, \omega_k$ in $(\mu_k,\mu_1)$ to be specified later. We also fix $\xi_k=1/I_1(\mu_k)$. 
Starting with $\cG_k^T$, we apply the second statement in Lemma \ref{concentration} for arm $k$ and arm $1$ with $n_0=(1+\epsilon)\xi_k \log r$ and $\xi=\omega_k$:
\begin{align*}
\bP (\cG_k^T) & \leq \sum_{r=1}^{T-1} \bP\left(\ell(r) = 1, k \in \cA_{r+1}, N_k(r) \geq (1+\epsilon) \xi_k \log T \right) \\
& \leq \sum_{r=1}^{T-1} \bP\left(N_1(r)\geq N_k(r), \hat Y_{k,N_k(r)} > \hat Y_{1, S_1^r(N_1(r),N_k(r))}, N_k(r) \geq (1+\epsilon) \xi_k \log T \right)\\
&\leq \frac{T}{1-e^{-I_1(\omega_k)}} e^{-(1+\epsilon)\xi_k I_1(\omega_k) \log T} + \frac{1}{1-e^{-I_{k}(\omega_k)}} e^{-(1+\epsilon)\xi_k I_{k}(\omega_k) \log T}\;.
\end{align*}
Similarly, we obtain
\begin{align*}
\bP (\cH_k^T) \leq \frac{T}{1-e^{-I_1(\omega)}} e^{-J_k I_1(\omega) \log T } + \frac{1}{1-e^{-I_{k}(\omega)}} e^{-J_k I_{k}(\omega) \log T}\;.
\end{align*}

Our objective is to bound $T\bP(\cH_k^T)$ and $\log (T) \bP(\cG_k^T)$ by constants. This is achieved for instance if $T\bP(\cH_k^T) \underset{T\to +\infty}{\longrightarrow} 0$ and $\log (T) \bP(\cG_k^T)\underset{T\to +\infty}{\longrightarrow} 0$. The following conditions are sufficient to ensure these properties:

\begin{itemize}
	\item $(1+\epsilon)\xi_k I_1(\omega_k)>1$
	\item $(1+\epsilon)\xi_k I_k(\omega_k)>0$
	\item $J_k I_1(\omega)>2$
	\item $J_k I_k(\omega)>1$
\end{itemize}

These conditions are met with the following values:
\begin{itemize}
	\item $\omega=\frac{1}{2}(\mu_1 + \max_{k \neq 1} \mu_k )$
	\item $J_k > \max(\frac{1}{I_k(\omega)}, \frac{2}{I_1(\omega)})$
	\item $\mu_k<\omega_k<\mu_1$ chosen such that $(1+\epsilon)I_1(\omega_k)>I_k(\omega_k)$. We are sure that such value exists if we choose $\omega_k$ close enough to $\mu_k$, thanks to the continuity of the rate functions and the fact that $I_k(\mu_k) = 0$ and $I_1(\mu_k) \neq 0$  (assumed in Assumption 1. of Theorem~\ref{th::log_regret_sda}).
\end{itemize}

Choosing these values, both the terms in $\cG_k^T$ and $\cH_k^T$ in \eqref{eq:dec} are part of the constant $C_{k}(\bm\nu,\epsilon)$ in Lemma~\ref{lem:dec}. We can now focus on upper bounding $\sum_{r=1}^{T-1} \bP (\cZ^r)$, which is more challenging.

\subsection{Upper Bound on $\sum_{r=1}^{T-1} \bP(\cZ^r)$}\label{app::boundZ}

The first steps of this part of the proof are again similar to \cite{SSMC}. The definition of the leader as the arm with the largest history gives the following property, that will be very useful for the analysis:
\begin{align*}
\ell(r) = k \Rightarrow N_k(r) \geq \left\lfloor \frac{r}{K}\right\rfloor -1
\end{align*}
So if an arm $k$ is the leader at a given round it has been drawn a linear amount of time at this round. Intuitively, this will provide very interesting concentration guarantees for the leader after a reasonable amount of rounds, that we are going to use in this section. For every $r \geq 8$, we define 
$a_r=\lfloor \frac{r}{4}\rfloor$ and use the decomposition
\begin{equation}\bP\left(\cZ^r\right)= \bP\left(\cZ^r \cap \cD^r\right) + \bP\left(\cZ^r \cap \bar\cD^r\right),\label{dec:Z}\end{equation}
where $\cD^r$ is the event that the optimal has been leader at least once in $[a_r,r]$: 
\[\cD^r = \{\exists u \in  [a_r, r] \text{ such that } \ell(u) = 1 \}.\]
We now explain how to upper bound the sum of the two terms in the left hand side of~\eqref{dec:Z}.

\subsubsection{Controlling $\bP(\cZ^r\cap\cD^r)$: arm 1 has been leader between $\lfloor r/4 \rfloor$ and $r$}

We introduce a new event 
\begin{align*}
	&\cB^{u} = \{\ell(u)=1, k \in \mathcal{A}_{u+1}, N_k(u)=N_1(u)-1 \text{ for some arm } k \}&
\end{align*}
If $\cD^r$ happens, then the event $\cZ^r$ can be true only if the leadership has been taken over by a sub-optimal arm at some round between $a_r$ and $r$, that is 
$$
\cZ^r \cap \cD^r \subset \cup_{u=a_r}^{r}\{\bar \cZ_u, \cZ_{u+1} \} \subset \cup_{u=a_r}^{r} \cB^u
$$
We now upper bound $\sum_{r=8}^{T-1} \sum_{u=a_r}^r \bP(\cB^u)$. We use the notation $b_r=\lfloor a_r/K \rfloor$, where we recall $a_r = \lfloor r/4 \rfloor$.  Then we write $\cB^u = \cup_{k=2}^K \cB_k^u := \left\{\ell(u)=1, k \in \mathcal{A}_{u+1}, N_k(u)=N_1(u)-1\right\}\}$, which fixes a specific suboptimal arm.
For any $w_k$ in $(\mu_k,\mu_1)$, one can write {\small
\begin{align}
&\sum_{r=8}^{T-1} \sum_{u=a_r}^r \bP(\cB_k^u)  = \bE \sum_{r=8}^{T-1} \sum_{u=a_r}^{r} \ind(\ell(u)=1) \ind(k \in \cA_{u+1}) \ind(N_1(u)=N_k(u)+1) \nonumber \\ 
& \leq \bE \sum_{r=8}^{T-1} \sum_{u=a_r}^r \ind(N_1(u) \geq  b_r) \ind(\bar Y_{k, N_k(u)} \geq \bar Y_{1, S_k^u(N_1(u), N_k(u))}) \ind(N_1(u)=N_k(u)+1) \ind(k\in \cA_{u+1}) \nonumber\\
& \leq \bE \sum_{r=8}^{T-1} \sum_{u=a_r}^r \ind(N_1(u) \geq  b_r) \ind( \bar Y_{k,N_k(u)} < w_k) \ind(N_1(u)=N_k(u)+1) \ind(k\in \cA_{u+1}) 
\label{eq:FirstPart}\\ 
& + \bE \sum_{r=8}^{T-1} \sum_{u=a_r}^r \ind(N_1(u) \geq  b_r) \ind( \bar Y_{1,S_k^{u}(N_1(u),N_k(u))} > w_k) \ind(N_1(u)=N_k(u)+1) \ind(k\in \cA_{u+1})\label{eq:SecondPart}
\end{align}}

We now separately upper bound each of these two terms. First, 
\begin{align*}
 \eqref{eq:FirstPart} & \leq  \bE\sum_{r=8}^{T-1} \sum_{u=a_r}^r \sum_{n_k = b_r-1}^{r} \ind(N_k(u) = n_k) \ind(k\in \cA_{u+1})\ind( \bar Y_{k,n_k} < w_k) \\
 & \leq  \bE\sum_{r=8}^{T-1}  \sum_{n_k = b_r-1}^{r} \ind( \bar Y_{k,n_k} < w_k) \underbrace{\sum_{u=a_r}^r\ind(N_k(u) = n_k) \ind(k\in \cA_{u+1})}_{\leq 1} \\
 & \leq \sum_{r=8}^{T-1}  \sum_{n_k = b_r-1}^{r} \bP( \bar Y_{k,n_k} < w_k) \\
 & \leq \sum_{r=8}^{T-1} \sum_{n_k = b_r-1}^{r} \exp\left(-n_k I_k(w_k)\right)  \\
 & \leq \frac{e^{(2+1/K)I_k(\omega_k)}}{(1-e^{-I_k(\omega_k)})(1-e^{-I_k(\omega_k)/4K})}
\end{align*}
Then, letting $[m,n]$ denote the set of subset of $[n]$ of size $m$, 
\begin{align*}
 \eqref{eq:SecondPart} & \leq \bE \sum_{r=8}^{T-1} \sum_{u=a_r}^r \sum_{n_k = b_r -1}^{r} \ind( \bar Y_{1,S_k^{u}(n_k +1, n_k)} > w_k) \ind(N_k(u) = n_k) \ind(k\in \cA_{u+1}) \\
 & \leq \bE \sum_{r=8}^{T-1} \sum_{u=a_r}^r \sum_{n_k = b_r -1}^{r} \sum_{\cS \in [n_k,n_k+1]} \ind( \bar Y_{1,\cS} > w_k) \ind (S_k^{u}(n_k +1, n_k) = \cS) \ind(N_k(u) = n_k) \ind(k\in \cA_{u+1}) \\
 & \leq \bE \sum_{r=8}^{T-1} \sum_{u=a_r}^r \sum_{n_k = b_r -1}^{r} \sum_{\cS \in [n_k,n_k+1]} \ind( \bar Y_{1,\cS} > w_k) \ind(N_k(u) = n_k) \ind(k\in \cA_{u+1}) \\
 & \leq \bE \sum_{r=8}^{T-1} \sum_{n_k = b_r -1}^{r} \sum_{\cS \in [n_k,n_k+1]} \ind( \bar Y_{1,\cS} > w_k) \underbrace{\sum_{u=a_r}^r \ind(N_k(u) = n_k) \ind(k\in \cA_{u+1})}_{\leq 1} \\
& \leq \sum_{r=8}^{T-1} \sum_{n_k = b_r -1}^{r} \sum_{\cS \in [n_k,n_k+1]} \bP( \bar Y_{1,\cS} > w_k)\\
& \leq \sum_{r=8}^{T-1} \sum_{n_k = b_r -1}^{r} (n_k+1) \bP( \bar Y_{1,n_k} > w_k) \\
& \leq \sum_{r=8}^{T-1} (r+1) \sum_{n_k = b_r -1}^{r} \exp\left(- n_k I_1(w_k)\right) \\
& \leq  \frac{e^{(2+1/K)I_1(\omega_k)}}{(1-e^{-I_1(\omega_k)})(1-e^{-I_1(\omega_k)/4K})^2}
\end{align*}
Here we have used that there are $n_k+1$ subsets in $[n_k,n_k+1]$ and that $\bP( \bar Y_{1,\cS} > w_k) = \bP( \bar Y_{1,n_k} > w_k)$ for all such subsets. Choosing $\omega_k$ such that $I_1(\omega_k)=I_k(\omega_k)$ (which is possible given the continuity assumptions on the two rate functions), we obtain
\begin{equation}\label{eq::bound_B}
\sum_{r=8}^{T-1}\bP\left(\cZ^r \cap \cD^r\right) \leq \sum_{r=8}^{T-1} \sum_{u=a_r}^r \bP(\cB^u) \leq \sum_{k=2}^K \frac{2e^{(2+1/K) I_1(\omega_k)}}{(1-e^{-I_1(\omega_k)})(1-e^{-I_1(\omega_k)/4K})^2}\;.
\end{equation}

\subsubsection{Controlling $\bP(\cZ^r\cap\cD^r)$: arm 1 has not been leader between $\lfloor r/4 \rfloor$ and $r$}

The idea in this part is to leverage the fact that if the optimal arm is not leader between $\lfloor s/4 \rfloor$ and $s$, then it has necessarily lost a lot of duels against the current leader at each round. We then use the fact that when the leader has been drawn "enough", concentration prevents this situation with large probability. We introduce
$$\cL^r= \sum_{u=s_0}^r \ind_{\cC^u}\,  $$
for the event $\cC^u= \{\exists a \neq 1, N_a(u)\geq N_1(u), \hat Y_{a,S_1^u(N_a(u),N_{1}(u))} \geq \hat Y_{1,N_1(u)}\}.$ One can prove the following inequality: 
\begin{align*}
\bP(\cZ^r \cap \bar \cD^r) \leq \bP(\cL^r\geq r/4)\;. 
\end{align*}
\begin{proof} Under $\bar \cD^r$ arm $1$ is a challenger for every round  $u \in [a_r, r]$. Then, each time $\cC^u$ is not true arm 1 wins its duel against the current leader and is pulled. Hence, if $\{ \cL_r<r/4 \}$ then we necessarily have $\{N_1(r)>r/2\}$ and arm $1$ is leader in round $r$. Hence, $\{\cZ^r \cap \bar \cD^r \} \cap \{ \cL_r<r/4 \} = \emptyset$, which justifies the inequality.
\end{proof}

Now, as in \cite{SSMC} we use the Markov inequality to get:
\begin{align*}
\bP(\cL^r\geq r/4) \leq \frac{\bE(\cL^r)}{r/4}= \frac{4}{r} \sum_{u=\lfloor r/4  \rfloor}^r \bP(\cC^u)\;.
\end{align*}
By further decomposing the probability of $\bP(\cC^u)$ in two parts depending on the value of the number of selections of arm $1$, we obtain the upper bound
\begin{align*}
	\bP(\cZ^r \cap \overline{\cD}^r) & \leq  \frac{4}{r}\sum_{u=\lfloor r/4 \rfloor}^r \bP \left( N_1(u) \leq (\log u)^2 \right) + \frac{4}{r}\underbrace{\sum_{u=\lfloor r/4 \rfloor}^r \bP\left(\cC^u,  N_1(u) \geq (\log u)^2 \right)}_{B_r}\;.
\end{align*}
We now upper bound the quantity $B_r$ defined above by using Lemma \ref{concentration}. For each $a$, for any  $\omega_a$ such that $\omega_a \in (\mu_a, \mu_1)$, one can write
\begin{align*}
B_r &\leq \sum_{u=\lfloor r/4 \rfloor}^r \bP(\cC^u, N_1(u)\geq (\log \lfloor r/4 \rfloor)^2)\\
& \leq \sum_{a=2}^K \sum_{u=\lfloor r/4 \rfloor}^r \bP(Y_{a,S_1^u(N_a(u),N_1(u))} > \hat Y_{1,N_1(u)}, N_1(u)\geq \log (\lfloor r/4 \rfloor)^2, N_a(u)>N_1(u)) \\
&\leq\sum_{a=2}^K \left( \frac{1}{1-e^{-I_1(\omega_a)}} e^{- (\log \lfloor r/4 \rfloor)^2 I_1(\omega_a)} + \frac{r}{1-e^{-I_k(\omega_a)}} e^{-(\log \lfloor r/4 \rfloor)^2 I_a(\omega_a)} \right)\;.
\end{align*}
Choosing each $\omega_a$ such that $I_1(\omega_a)=I_a(\omega_a)$, we obtain: 
\begin{align*}\label{eq::bound_bis_C}
\begin{split}
\frac{4}{r}\sum_{r=8}^T  B_r & \leq \sum_{r=8}^T \sum_{a=2}^K \frac{4(r+1)}{r(1-e^{-I_1(\omega_a)})}e^{-(\log \lfloor r/4 \rfloor)^2 I_a(\omega_a)}\\
& \leq  \sum_{a=2}^K\sum_{r=8}^T \frac{6}{1-e^{-I_1(\omega_a)}} e^{-(\log \lfloor r/4 \rfloor)^2 I_a(\omega_a)},
\end{split}
\end{align*}
and for each $a$ the series in $r$ is convergent as for any constant $C$, $C\log(r) \leq (\log \lfloor r/4 \rfloor)^2$ for $r$ large enough. Hence, there exists some constant $D(\bm\nu)$ where $\bm\nu = (\nu_1,\dots,\nu_K)$ such that $\frac{4}{r}\sum_{r=8}^T  B_r \leq D(\bm\nu)$. It follows that 
\begin{align*}
	\sum_{r=8}^{T}\bP(\cZ^r \cap \overline{\cD}^r) & \leq  \sum_{r=8}^{T}\frac{4}{r}\sum_{u=\lfloor r/4 \rfloor}^r \bP \left( N_1(u) \leq (\log u)^2 \right) + D(\bm\nu)\;.
\end{align*}
We now transform the double sum in the right-hand side into a simple sum by counting the number of times each term appears in the double sum:
$$\sum_{r=8}^T\frac{4}{r}\sum_{u=\lfloor r/4 \rfloor}^r \bP \left( N_1(u) \leq (\log u)^2 \right) = \sum_{r=8}^T \left( \sum_{t=1}^r \frac{4}{t} \ind(t \in [r, 4r]) \right) \bP(N_1(r)\leq (\log r)^2)\;.$$ 
Noting that $\sum_{t=1}^r \frac{4}{t} \ind(t \in [s, 4s]) \leq (4s-s+1)\times \frac{4}{s}\leq 16$, we finally obtain:
\begin{equation}\label{eq:termD}
\sum_{r=8}^{T}\bP(\cZ^r \cap \overline{\cD}^r)  \leq  16 \sum_{r=1}^{T} \bP\left(N_1(r) \leq (\log(r))^2\right) + D(\bm\nu). 
\end{equation}

Combining \eqref{eq::bound_B} and \eqref{eq:termD} yields 
\[\sum_{r=1}^{T} \bP\left(\cZ^r\right) \leq 16 \sum_{r=1}^{T} \bP\left(N_1(r) \leq (\log(r))^2\right) + D'_k(\bm\nu)\]
for some constant $D'_k(\bm\nu)$ that depends on $k$ and $\bm\nu$, which contributes to the final constant $C_k(\bm\nu,\epsilon)$ in Lemma~\ref{lem:dec}. Plugging this inequality in Equation~\eqref{eq:dec} concludes the proof of Lemma~\ref{lem:dec}. 

%% file: appendix_D2.tex
\section{Probability that the Optimal Arm is not Drawn Enough: Proof of Lemma~\ref{lemma::decomposition}}\label{app::lem_decomposition}

We start with a decomposition that follows the steps of \cite{BESA} for BESA with 2 arms that we generalize for $K$ arms.

We first denote by $r_j$ the round of the $j^{th}$ play of arm $1$ with $r_0=0$ and let $\tau_j= r_{j+1}-r_j$. We notice that $\tau_0\leq K$ as all arms are initialized once. Then:

\begin{align*}
\bP \left( N_1(r) \leq (\log r)^2 \right) & \leq \bP \left( \exists j \in \{1,..., \log r^2 \} : \tau_j\geq r/(\log r)^2 - 1 \right) \\
& \leq \sum_{j=1}^{(\log r)^2}\bP \left(\tau_j \geq r/(\log r)^2-1 \right)
\end{align*}

\begin{proof} If we assume that $\forall j$ $\tau_j \leq r/(\log r)^2 -1$ then $t_{\log r^2} = \sum_{j=0}^{\log r^2} \tau_j < r$, which yields $N_\ell(r)> \log r^2 +1$.\end{proof}

We now fix $j \leq (\log r)^2$ and upper bound the probability of the event \[\cE_j := \{\tau_j \geq r/\log r^2 -1\}\;.\] On this event arm $1$ lost at least $r/\log r^2$ consecutive duels between $r_j+1$ and $r_{j+1}$ (either as a challenger of as the leader) which yields
\begin{align*}
\bP (\cE_j)\leq \bP \left(\forall s \in \{ r_j+1,..., r_j + \lfloor r/\log r^2 -1 \rfloor \} :\right.& \{\hat Y_{1,j} \leq \hat Y_{\ell(s), S_{1}^s(N_{\ell(s)}(s),j)}, N_1(s)=j, N_{\ell(s)}(s) \geq j\} \\ & \left.\cup \{\ell(s)=1 , N_1(s)=j\}\right) 
\end{align*}
The important change compared to the proof of \cite{BESA} is that with $K>2$, 1) we don't know the identity of the leader and 2) the leader is not necessarily pulled if it wins its duel against 1.

Now we notice that when $r$ is large, the time range considered in $\cE_j$ is large. By looking at the second half of this time range only, we can ensure that the leader has been drawn a large number of times. More precisely, introducing the two intervals
\begin{align*}\cM_{j, r}^1= & \left[r_j+1, r_j+ \left\lfloor \frac{r/\log r^2 -1}{2}\right\rfloor\right] \\
	\cM_{j, r}^2= & \left[t_j+ \left\lceil \frac{r/\log r^2 -1}{2}\right\rceil, t_j+\left\lfloor r/\log r^2 \right\rfloor -1\right]\,
\end{align*}
it holds that
\[
\bP (\cE_j) \leq \bP (\forall s \in \cM_{j, r}^2 : \{
\hat Y_{1,j} \leq \hat Y_{\ell(s), S_1^s(N_{\ell(s)}(s),j)}, N_1(s)=j, N_{\ell(s)}(s)\geq j \} \cup \{\ell(s)=1, N_1(s)=j\})\;.
\]
But we know that on $\cM_{j,r}^2$ the leader must has been selected at least $\frac{1}{K}\left(j+ \left\lceil \tfrac{r/\log r^2 -1}{2}\right\rceil \right)$ times. Let $r_K$ be the first integer such that $\log^2(r) < \frac{1}{K-1}\left\lceil \tfrac{r/\log r^2 -1}{2}\right\rceil$, for every $r \geq r_K$, as $j \leq \log^2(r)$, the leader has been selected strictly more than $j$ times, which prevents arm 1 from being the leader for any round in $\cM_{j,r}^2$. Hence, for $r \geq r_K$, for all $j \leq \log^2(r)$,
\[
\bP (\cE_j) \leq \bP \left(\forall s \in \cM_{j, r}^2 : \{
\hat Y_{1,j} \leq \hat Y_{\ell(s), S_1^s(N_{\ell(s)}(s),j)}, N_1(s)=j, N_{\ell(s)}(s)\geq j \}\right)\;.
\]

To remove the problem of the identity of the leader we would like to find a way to fix our attention on one arm. To this extent, we notice that during an interval of length $|\cM_{j, r}^2|$, if there are only $K-1$ candidates for the leader then one of them must have been leader at least $m_r :=|\cM_{j, r}^2|/(K-1)-1$ times during this range. We also know that at any round in $\cM_{j,r}^2$, the leader satisfies $N_{\ell(s)}(s) \geq (t_j+ \lfloor \frac{r/\log r^2 -1}{2}\rfloor)/K-1 \geq (\lfloor \frac{r/\log r^2 -1}{2}\rfloor)/K-1 =\frac{|\cM_{j,r}^1|}{K}-1 :=c_r$. Observe that $m_r>c_r$. Finally, we introduce the notation 
$$
I_{j, r}^k = \{s \in \cM_{j, r}^2 : \ell(s) = k \}
$$
for the set of rounds in $\cM_{j,r}^2$ in which a particular arm $k$ is leader. From the above discussion, we know that there exists an arm $k$ such that $|I_{j,r}^k| \geq m_r$.  

To ease the notation, we introduce the event 
$$
\cW_{s,j}^{k} = \left\{\left\{ \hat Y_{1, j} < \hat Y_{k,S_1^s(N_k(s), j)} \right\}, N_k(s)\geq c_r, N_1(s)=j\right\}
$$
and write 
\begin{align*}
\bP(\cE_j) &\leq \bP\Bigg(\bigcap_{s \in \cM_{j,r}^2} \bigcup_{k=2}^K \left\{\ell(s)=k, 1 \notin \cA_s) \right\}\Bigg) \\
& \leq \bP\Bigg(\bigcap_{k=2}^K \bigcap_{s \in I_{j,r}^k} \cW_{s, j}^k\Bigg) \\
& \leq \bP\left( \bigcup_{k=2}^K \left\{|I_{j, r}^k|>m_r, \bigcap_{s \in I_{j,r}^k} \cW_{s, j}^k \right\} \right) \\
& \leq \sum_{k=2}^K \bP\left(|I_{j, r}^k|>m_r, \bigcap_{s \in I_{j,r}^k} \cW_{s, j}^k \right)\;. 
\end{align*}

Finally, we define for any integer $M$ the event that we can find $M$ pairwise non-overlapping sub-samples in the set of the sub-samples of arm $k$ drawn in rounds $s \in I_{j,r}^k$:
$$
\cF_{j, M}^{k, r} = \left\{\exists i_1,..., i_M \in I_{j, r}^k: \forall m<m' \in [M], S_{1}^{i_m}(N_{k}(i_m),j) \cap S_{1}^{i_{m'}}(N_{k}(i_{m'}),j) = \emptyset \right\}
$$
Introducing  $H_{j,r}^k = \min_{s \in I_{j,r}^k} N_k(s)$, the minimal size of the history of arm $k$ during rounds in $I_{j,r}^k$ (which is known to be larger than $c_r$ as $k$ is leader in these rounds), one has 
\begin{align}
&\bP(\cE_j)  \leq \sum_{k=2}^K \bP\left(|I_{j, r}^k|>m_r, \cap_{s \in I_{j,r}^k} \cW_{s,j} \cap \{\cF_{j, M}^{k, r} \cup \bar \cF_{j, M}^{k, r} \} \right) \nonumber \\
& \leq \sum_{k=2}^K \bP\left(|I_{j, r}^k| \geq m_r, H_{j, r}^k \geq c_r, \bar \cF_{j, M}^{k, r} \right) + \sum_{k=2}^K\bP \left(|I_{j, r}^k|>m_r, \cap_{s \in I_{j,r}^k} \cW_{s,j} \cap \cF_{j, M}^{k, r} \right) \label{eq:TwoCrucialTerms}
\end{align}

\paragraph{Upper bound on the first term in \eqref{eq:TwoCrucialTerms}} The probability 
$\bP\left(|I_{j, r}^k| \geq m_r, H_{j, r}^k \geq c_r, \bar \cF_{j, M}^{k, r} \right)$ can be upper bounded by 
\[\bP\left(\left. \# \left\{ \text{pairwise non-overlapping subsets in } \left(S_1^s(N_k(s),j)\right)_{s \in I_{j,r}^k} \right\} < M \right| \left\{|I_{j, r}^k|>m_r, H_{j, r}^k \geq c_r\right\} \right)\;.\]
This probability can be related to some intrinsic properties of the sampler $\text{SP}(H,j)$. To formalize this, we introduce the following definition.

\begin{definition}\label{def:X} For every integers $N,H,j$ such that $H>j$, $X_{N, H, j}$ is a random variable which counts the maximum number of non-overlapping subsets among $N$ i.i.d. samples from $\text{SP}(H,j)$.
\end{definition}

Letting $H_1,\dots,H_{m_r}$ be integers that are all larger than $c_r$, and letting $S_1,\dots,S_{m_r}$ be independent subsets such that $S_i \sim \text{SP}(H_i,j)$, the above probability is upper bounded by 
\[\bP\left(\# \left\{ \text{pairwise non-overlapping subsets in } \left(S_i\right)_{i=1}^{m_r} \right\} < M \right)\]
which is itself upper bounded by 
$\bP\left(X_{m_r, c_r, j} < M \right)$.

This last inequality is quite intuitive: if one draws subsets of size $j$ from histories that may be larger than $c_r$, there is more ``room'' for non-overlapping subsets than if we always draw them from the same history of size $c_r$. For Random Block sampling, where the drawn subset is fully determined by the random position of its first element, to formalize this intuition it is sufficient to prove that if $X_i, Y_i$ are two sequences of random variables such that  $X_i$ is uniform in $[H_i-j]$ and $Y_i$ is uniform in $[H - j]$, where $H_i \geq H$, the random variable that counts the maximal number of elements in the sequence $(Y_i)$ whose pairwise distance are larger than $j$ is stochastically dominated by that the same random variable but for the sequence $(X_i)$. We performed numerical experiments that confirm that this last condition holds. 

\paragraph{Upper bound on the second term in \eqref{eq:TwoCrucialTerms}} On the event $\left(|I_{j, r}^k|>m_r, \cap_{s \in I_{j,r}^k} \cW_{s,j} \cap \cF_{j, M}^{k, r} \right)$, one can define $\tilde{\imath}_1,\dots,\tilde{\imath}_M$ the  first $M$ rounds in $I_{j, r}^k$ for which the subsets $\tilde{S}_{m}:=S^{\tilde\imath_m}(N_k(\tilde\imath_m),j)$ are pairwise non-overlapping and we get 
\[\bP \left(|I_{j, r}^k|>m_r, \cap_{s \in I_{j,r}^k} \cW_{s,j} \cap \cF_{j, M}^{k, r} \right) \leq \bP\left(\forall m \in [M], \hat{Y}_{1,j} \leq \hat Y_{k,\tilde S_m}\right)\;.\]
By definition the subsets $\tilde{S}_m$ are pairwise non-overlapping, hence the sub-samples $\hat Y_{k,\tilde S_m}$ are independent. We prove that this probability can be in fact upper bound by the \textit{balance function} we defined in section~\ref{sec:regret}.

Indeed, introducing $X \sim \nu_{1,j}$ and an independent i.i.d. sequence $Z_i \sim \nu_{k,j}$, one can write
\begin{align*}
\bP \left(|I_{j, r}^k|>m_r, \cap_{s \in I_{j,r}^k} \cW_{s,j} \cap \cF_{j, M}^{k, r} \right)& \leq  \bP(X< \min_{i\in [M]} Z_i)  \\
& =\bE_{\substack{X \sim \nu_{1,j}\\Z \sim \nu_{k,j}^{\otimes j}}} \left[\prod_{i=1}^{M} \ind_{X \leq Z_i} \right] \\
& = \bE_{X \sim \nu_{1,j}} \left[\bE_{Z \sim \nu_{k,j}^{\otimes j}} \left[\left.\prod_i \ind_{X \leq Z_i} \right|X\right] \right] \\
& = \bE_{X \sim \nu_{1,j}} \left[\left(1-F_{k,j}(X) \right)^M \right] \\
& = \alpha_k(M,j). 
\end{align*}

\paragraph{Conclusion} Putting things together, we have proved that 
$$
\bP(\cE_j) \leq  (K-1) \bP \left(X_{m_r, c_r, j} < M \right) + \sum_{k=2}^K \alpha_k(M, j),
$$
where $X_{N,H,j}$ and $\alpha_k(M,j)$ are introduced in Definition~\ref{def:X} and \ref{def:balance} respectively. If we replace $M$ by the sequence $\beta_{r,j}$ we have 
\begin{align*}
\sum_{r=1}^T \bP(N_1(r)\leq \log r^2)  &\leq r_K + \sum_{r=r_K}^T \sum_{j=1}^{\log r^2}\left[(K-1) \bP \left(X_{m_r, c_r, j} < \beta_{r,j} \right) + \sum_{k=2}^K \alpha_k(\beta_{r,j}, j)\right] \\
&\leq r_K + \sum_{r=r_K}^T \sum_{j=1}^{\log r^2}\left[(K-1) \bP \left(X_{c_r, c_r, j} < \beta_{r,j} \right) + \sum_{k=2}^K \alpha_k(\beta_{r,j}, j)\right]
\end{align*}

as $c_r\leq m_r$, which proves Lemma~\ref{lemma::decomposition}.

This definition allows to analyze separately the properties of the sub-sampling algorithms and the properties of the distribution family for randomized samplers.

%% file: appendix_E.tex
\section{Proof that \RBSSDA{} Satisfies the Diversity Property}\label{app::diversity_rbsda}

We recall that $X_{m,H,j}$ denotes the maximal number of pairwise non-overlapping subsets obtained in $m$ i.i.d. samples from $\mathrm{RB}(H,j)$. 
In this section we aim at upper bounding the probability of 
\[\bP\left(X_{m,H,j} \leq \gamma r/(\log r)^2\right)\] for some values of $m$, $H$, $j$, that will be fixed later.
This probability depends on several parameters, with straightforward effects:
\begin{itemize}
	\item The probability decreases with the length of the history size $H$.
	\item The probability increases with the size $j$ of each sub-sample.
	\item The probability decreases with the total number of sub-samples we draw $m$. Intuitively if $m$ is large enough every sample of size $j$ in the history will be drawn.  
\end{itemize}

\paragraph{First step with $j=1$:} in this case the distribution of the $m$ subsets of size 1 is actually the distribution of sampling with replacement in $H$. The question of the number of different items drawn with sampling without replacement has been studied in \cite{items_count}, from which we use the following result: 

\textbf{Result 1}: for any $k \in [H]$, $\bP(X_{m, H, 1}=k)=\frac{H!}{(H-k)! \times H^{m}} \times S_{k,m}$, where $S_{k,m}$ is the Stirling number of the second kind for $k, m$.\\

We use this result with further assumptions that are specific to our problem and will ease the computation: $H=m=O(r/(\log r)^2)$. To ease the notation we continue to use $H$, and write $\gamma t/(\log t)^2=\alpha H$ for some $\alpha \in (0,1)$. 

We first look at $\bP(X_{H, H, 1} = \alpha H)$. According to \cite{bound_stirling} the following inequality holds
$$
S_{k,H} \leq \frac{1}{2} {H \choose k} k^{H-k}\;,
$$
This allows to upper bound the expression in result 1: 
$$
\bP(X_{H, H, 1} = k) \leq \frac{1}{2} \left(\frac{k}{H} \right)^{H-k} {H \choose k}
$$

We now want to bound ${H \choose k}$. As $k$ is small compared with $H$, it is natural to use
$$
{H \choose k} \leq \frac{H^k}{k!}
$$

We then bound $1/k!$ by its Stirling approximation and add a multiplicative constant $c$ along the way:
$$
{H \choose k} \leq c \frac{H^k}{\sqrt{2\pi k} \times k^k} e^k
$$

Refactoring provides
$$
\bP(X_{H, H, 1}=k) \leq \frac{c}{2} \left(\frac{k}{H} \right)^{H-2k} \frac{e^k}{\sqrt{2\pi k}}
$$

Then we notice that if $k \leq H-2k$ we get: 
$$
\bP(X_{H, H, 1}=k) \leq \frac{c}{2\sqrt{2\pi k}} \left(\frac{ke}{H} \right)^{H-2k} 
$$

Now we can replace $k$ by $\alpha H$ (assume it's an integer for the simplicity of notations), such that 1) $\alpha \leq \frac{1}{3} \Rightarrow H(1-3\alpha)>0$ and $\alpha e<1$.

If $k \leq \alpha H$, $\alpha e\leq 1$ and $(1-2\alpha)>0$ then: $\left(\frac{ke}{H} \right)^{H-2k} \leq (\alpha e)^{H-2k}$. We have:
\begin{equation}\label{eq::div_rbs}
\begin{split}
\bP(X_{H, H, 1}\leq \alpha H) & \leq \frac{c}{2\sqrt{2\pi}} \sum_{k=0}^{\lfloor \alpha H \rfloor} \left(\alpha e\right)^{H-2k} \\
& \leq \frac{c}{2\sqrt{2\pi}}  \sum_{k=0}^{\lfloor \alpha H \rfloor} (\alpha e)^{H-2(\lfloor \alpha H \rfloor -k))}\\
& \leq \frac{c}{2\sqrt{2\pi}}(\alpha e)^{H-2\lfloor \alpha H \rfloor} \frac{1}{1-(\alpha e)^2}\\
& \leq \frac{c}{2\sqrt{2\pi}}  \frac{1}{1-(\alpha e)^{2}} \exp\left(-(1-2\alpha)H \log(1/(\alpha e))\right)
\end{split}
\end{equation}

\paragraph{From $X_{H,H,1}$ to $X_{H,H,j}$} This result is enough to get general properties for Random Block Sampling. Indeed, as the process of RBS consists in only drawing the first element of the block used in the duel we can see that the previous bound also applies to the number of unique starting points. With this property, the Random Block Sampler satisfies for all $x>0$:
\[\bP\left(X_{m,H,j} \leq \left\lfloor \frac{x}{j}\right\rfloor \right) \leq \bP\left(X_{m,H,1}\leq x\right) \]
\begin{proof}
Assume that the Random Block Sampler provides $x$ blocks with different starting points. Let's further assume that $x$ is an integer and try to identify the sequence of starting times $t_i =(t_1,...,t_x)$ that minimizes the number of mutually non-overlapping samples: the value of $t_1$ is not important due to the symmetry of the problem. Then if we want to reduce the possibilities to get non-overlapping sample we want to choose a value for $t_2$ that 1) makes the blocks $[t_1, t_1+j]$ and $[t_2, t_2+k]$ non-overlapping and 2) makes things easier to continue this process for $t_3,...,t_m$. It seems intuitive to choose either the block starting at $t_1+1$ or at $t_1-1$ as we cover the minimum amount of space with the constraint that $t_2 \neq t_1$. If we repeat this choice until $m$ blocks are chosen and reorder the blocks properly, we get a sequence of starting points $[t_1, t_1+1,...t_1+m]$ that are all different and minimize the total amount of space covered by the block. Even in this setup, we can find exactly $\left\lfloor \frac{m}{j} \right\rfloor$ mutually non-overlapping blocks as for instance all $[t_1+kj, t_1+(k+1)j-1]$, $[t_1+k'j, t_1+(k'+1)j-1]$ blocks are non-overlapping for $k\neq k'$ and $(k, k') \in [0, \left\lfloor \frac{m}{j} \right\rfloor - 1]$.
\end{proof}

We can finally prove the following for Random Block sampling. 

\begin{lemma}[Diversity Property for Random Block Sampling]
	\label{BS}
	If we choose a constant $\gamma \leq 1/3\times \frac{1}{2K}$ then Random Block Sampling satisfies the diversity property. 
\end{lemma}

\begin{proof}
For $\gamma \leq \left\lfloor 1/3\times \frac{1}{2K}\right\rfloor$, there exists $\alpha>0$ such that:
\begin{align*}
\bP(X_{c_r, c_r, j}\leq \gamma /j (r/(\log r)^2)) &\leq \bP(X_{c_r, c_r, j}\leq \alpha/j c_r) & \\
& \leq \bP(X_{c_r, c_r, 1}\leq \alpha c_r) \\
& = o(r^{-2})
\end{align*}
The last line comes from the expression obtained in Equation~\eqref{eq::div_rbs}, and allows to conclude that $\sum_{r=1}^T\sum_{j=1}^{(\log r)^2}\bP(X_{c_r, c_r, j}\leq \alpha/j (r/(\log r))^2)=o(\log T)$
\end{proof}

%% file: appendix_F.tex
\section{Analysis of the Balance Function for Some Distributions}\label{app::balance}

For the simplicity of the notation we write the balance function $\alpha(M,j)$ for any distribution and any instance of these distributions. The family of distributions and the notation for their parameter will always mentioned at the beginning of the corresponding sub-section.

In the next parts we use the notation $G(x)=1-F(x)$ where $F$ denotes the CDF of some distribution. For some arm distribution $\nu_i$ the distribution of the sum of $j$ independent observations drawn from $\nu_i$ is denoted by $\nu_{i, j}$. With this notation, for two arms 1 and 2 we write:
$$
\alpha(M,j) = \bE_{Z \sim \nu_{1, j}}(G_{2,j}(Z)^M)
$$

\subsection{The Bernoulli Distribution is Balanced}

We prove the following lemma, which bears strong similarity with an upper bound given by \cite{TS_Emilie} for a similar quantity in their analysis of Thompson Sampling.

\begin{lemma}[Bound on $\alpha(M, j)$]\label{app:balance_Blem}
	\label{bernoulli_bound}
	
	For two Binomial distributions $\nu_{1} \sim \mathcal{B}(j, \mu_1)$ and $\nu_{2} \sim \mathcal{B}(j, \mu_2)$ such that $\mu_1>\mu_2$ and for any integer $M>1$: $\exists \lambda >1$ such as
	$$
	\bE_{X \sim \nu_{1,j}} \left(\left(1-F_{j, \mu_2}(X)\right)^M \right) \leq C_{\lambda_0, \lambda} \frac{1}{M^\lambda}e^{-j d_{\lambda, \mu_1, \mu_2}} + \left(\frac{1}{2} \right)^M
	$$
	
	Where $C_{\lambda, \mu_1, \mu_2}>0$, and $F_{j, \mu_2}$ is the CDF of a Binomial $\mathcal{B}(j, \mu_2)$.
	
\end{lemma}

\begin{proof}
	We use the same notation as before: $G_2(k) = 1 - F_2(k)$ and $f_1, f_2$ as the PMF of $\nu_1, \nu_2$.
	
	We first use a common property of Binomial distributions, $\forall k>\lceil j\mu_2 \rceil$: $G(k)\leq \frac{1}{2}$. So we can directly write:
	$$
	\bE_{X \sim \nu_1} \left(\left(1-F_{j, \mu_2}(X)\right)^M \right) \leq \left(\frac{1}{2} \right)^M + \underbrace{\sum_{k=0}^{\lfloor j \mu_2 \rfloor} f_1(k) G_2(k)^M}_{(A)} 
	$$
	
	Using convexity we get: $G(k)^M \leq \exp \left(-M F_2(k) \right)$, hence 
	$$
	(A) \leq \sum_{k=0}^{\lfloor j \mu_2 \rfloor} f_1(k) \exp \left(-M F_2(k) \right)
	$$
	
	Then we use that for $\lambda>1$, $\forall x>0$: $x^\lambda e^{-x} \leq \left(\frac{\lambda}{e}\right)^\lambda=C_{\lambda}$, so:
	$$
	(A) \leq \frac{C_\lambda}{M^\lambda} \sum_{k=0}^{\lceil j \mu_2 \rceil} \frac{f_1(k)}{F_2(k)^\lambda} \leq \frac{C_\lambda}{M^\lambda} \sum_{k=0}^{\lceil j \mu_2 \rceil} \frac{f_1(k)}{f_2(k)^\lambda}
	$$
	
	As in \cite{TS_Emilie}, we compute:
	\begin{align*}
	\frac{f_1(k)}{f_2(k)^\lambda} &\leq \frac{\mu_1^k (1-\mu_1)^{j-k}}{\mu_2^{\lambda k}(1-\mu_2)^{\lambda(j_k)}} \\
	& \leq \left(\frac{1-\mu_1}{(1-\mu_2)^\lambda} \right)^j \left(\frac{\mu_1(1-\mu_2)^\lambda}{\mu_2^{\lambda}(1-\mu_1)} \right)^k\\
	& = \left(\frac{1-\mu_1}{(1-\mu_2)^\lambda} \right)^j R_\lambda(\mu_1, \mu_2)^k
	\end{align*}
	
	with $R_\lambda(\mu_1, \mu_2)=\frac{\mu_1(1-\mu_2)^\lambda}{\mu_2^{\lambda}(1-\mu_1)}$. We then notice that we can choose $\lambda >1$ such that $R_\lambda(\mu_1, \mu_2)>1$. It is true for any $\lambda>1$ if $\mu_2\leq 0.5$, and for $1<\lambda< \log \frac{\mu_1}{1-\mu_1}/\log \frac{\mu_2}{1-\mu_2}$ if $\mu_2>0.5$.
	
	Plugging that expression into the sum gives:
	\begin{align*}
	(A) \leq & \frac{C_\lambda}{M^\lambda} \sum_{k=0}^{\lceil j \mu_2 \rceil} \frac{f_1(k)}{f_2(k)^\lambda} \\
	\leq & \frac{C_\lambda}{M^\lambda} \left(\frac{1-\mu_1}{(1-\mu_2)^\lambda} \right)^j \sum_{k=0}^{\lceil j \mu_2 \rceil} R_\lambda(\mu_1, \mu_2)^k \\
	= & \frac{C_\lambda}{M^\lambda} \left(\frac{1-\mu_1}{(1-\mu_2)^\lambda} \right)^j \frac{R_\lambda(\mu_1, \mu_2)^{\lfloor j\mu_2 \rfloor +1}-1}{R_\lambda(\mu_1, \mu_2)-1} \\
	\leq & \frac{C_\lambda}{M^\lambda} \left(\frac{1-\mu_1}{(1-\mu_2)^\lambda} \right)^j \frac{R_\lambda(\mu_1, \mu_2)}{R_\lambda(\mu_1, \mu_2)-1} R_\lambda(\mu_1, \mu_2)^{j\mu_2} \\
	= & \frac{C_\lambda}{M^\lambda} \frac{R_\lambda(\mu_1, \mu_2)}{R_\lambda(\mu_1, \mu_2)-1} \left(\frac{1-\mu_1}{(1-\mu_2)^\lambda} \right)^{j(1-\mu_2)} \left(\frac{\mu_1}{\mu_2^\lambda} \right)^{j\mu_2} \\
	=& \frac{C_\lambda}{M^\lambda} \frac{R_\lambda(\mu_1, \mu_2)}{R_\lambda(\mu_1, \mu_2)-1} e^{-j d_\lambda(\mu_2, \mu_1)}\\
	=& \frac{C_{\lambda, \mu_1, \mu_2}}{M^\lambda} e^{-j d_\lambda(\mu_2, \mu_1)}
	\end{align*}
	
	where $d_\lambda(\mu_2, \mu_1) = \lambda \left(\mu_2 \log \mu_2 + (1-\mu_2) \log(1-\mu_2) \right) - \left(\mu_2 \log \mu_1 + (1-\mu_2)\log (1-\mu_1) \right)= \text{KL}(\mu_2, \mu_1)-(\lambda-1)\text{H}(\mu_2)$, $\text{KL}(\mu_2, \mu_1)$ denotes the KL-divergence between $\nu_2$ and $\nu_1$, and $\text{H}(\mu_2)=\bE_{X \sim \nu_{2,j}}(\log f_2(X))$. We need to choose $\lambda$ as: 
	$$
	\lambda < 1 + \frac{\text{KL}(\mu_2, \mu_1)}{\text{H}(\mu_2)} = \lambda_0(\mu_1, \mu_2)
	$$
	
	Note that those quantities correspond to the Bernoulli distributions, the $j$ is not involved here. In \cite{TS_Emilie}, the authors explain that this condition is more restrictive than the previous one so we can state that $\forall \lambda < \lambda_0(\mu_1, \mu_2)$:
	$$
	\bE_{X \sim \nu_{1,j}} \left(\left(1-F_{j, \mu_2}(X)\right)^M \right) \leq \left(\frac{1}{2} \right)^M + \frac{C_{\lambda, \mu_1, \mu_2}}{M^\lambda} e^{-j d_\lambda(\mu_2, \mu_1)}
	$$
	
\end{proof}

This is enough to prove that the Bernoulli distribution is balanced by replacing $M$ by $\lfloor \beta t /(\log t)^2 \rfloor$ in the expression in Lemma~\ref{app:balance_Blem} and summing on $t$ and $j$. The power term is in $o(t(\log t^2))$, while the other term is the term of a convergent geometric series in $j$ multiplied by a term in $o(1/t)$ in $t$, which is enough to get the result.

\subsection{The Poisson Distribution is Balanced}
We can actually use the same sketch of proof as for Bernoulli distributions, using that for 2 Poisson random variables:
\begin{align*}
\frac{p_{1,j}(k)}{p_{2,j}(k)^\lambda} = & e^{-j(\theta_1-\lambda \theta_2)} \left(\frac{k!}{n^k} \right)^\lambda \left(\frac{\theta_1}{\theta_2^\lambda} \right)^k \leq e^{-j(\theta_1-\lambda \theta_2)} \left(\frac{\theta_1}{\theta_2^\lambda} \right)^k
\end{align*}
So:
\begin{align*}
\sum_{k=0}^{d_{0,j}}\frac{p_{1,j}(k)}{p_{2,j}(k)^\lambda} &
\leq  e^{-j(\theta_1-\lambda \theta_2)} \sum_{k=0}^{d_{0,j}} \left(\frac{\theta_1}{\theta_2^\lambda} \right)^k\\
& \leq \frac{\theta_2^\lambda}{|\theta_1-\theta_2^\lambda|}e^{-j(\theta_1-\lambda \theta_2)} \times \max \left\{1, \left(\frac{\theta_1}{\theta_2^\lambda}\right)^{d_{0,j}} \right\}\;,
\end{align*}
where $d_{0,j}=\theta_2 j -1$. Now we remark that if we choose $\lambda \in (1, \theta_1/\theta_2)$ we have 2 possibilities: 1) we can choose $\lambda$ such that the second term equals one, hence we can bound the whole term by a constant without further conditions, or 2) $\forall \lambda \in (0, \theta_1, \theta_2)$: $\left(\frac{\theta_1}{\theta_2^\lambda} \right)>1$. Let us focus on the second case, we study the term $e^{-j\left((\theta_1-\lambda \theta 2)-\theta_2(\log \theta_1 - \lambda \log \theta_2) \right)}$. As for Bernoulli distributions, we identify the KL-divergence between $\nu_2$ and $\nu_1$ and write:
$$
(\theta_1-\lambda \theta 2)-\theta_2(\log \theta_1 - \lambda \log \theta_2) = \text{KL}(\nu_2, \nu_1) - (\lambda-1) \theta_2(1 -  \log \theta_2)
$$
So if $\log \theta_2 > 1$ we can choose any $\lambda>1$. In the other case we have to restrict our choice of $\lambda$ to get:
$$
\lambda < 1+\frac{\text{KL}(\nu_2, \nu_1)}{\theta_2(1-\log(\theta_2))}
$$
So with an appropriate choice for $\lambda$ Poisson distributions are balanced with the same argument that makes Bernoulli distributions balanced.

\subsection{The Gaussian Distribution is Balanced}

For the Gaussian distribution we leverage the fact that both the PDF and CDF of any Gaussian distribution can be expressed with the PDF and CDF of the standard normal distribution. With such decomposition, we can express $\alpha(M,j)$ as a function of these CDF/PDF and use some properties of the normal distribution.

We use the notations $f$ and $G$ for the PDF and CDF of the $\mathcal{N}(0,1)$ distribution, $\Delta$ for the gap between the two arms, and compute the expectation:
\begin{align*}
\alpha(M, j) &= \int_{-\infty}^{+\infty} f_{1, j}(x)G_{2,j}(x)^M dx \\
& \leq \int_{-\infty}^{z} f_{1, j}(x)G_{2,j}(x)^M dx + G_{2,j}(z)^M \text{,  }\forall z \in \R \\ 
& \leq \int_{-\infty}^{z} f\left(\frac{x-\mu_{1, j}}{\sqrt{j}}\right)G\left(\frac{x-\mu_{2,j}}{\sqrt{j}}\right)^M dx + G_{2,j}(z)^M \\ 
& \leq \int_{-\infty}^{\frac{z-\mu_{2,j}}{\sqrt{j}}} f\left(y-\sqrt{j}\Delta\right)G(y)^M dy + G_{2,j}(z)^M \\ 
\end{align*}
At this step we use two things: 1) the normal distribution satisfies $f(x-a)=e^{-a^2+2ax}f(x)$ for all $a, x$, and 2) $h: x \rightarrow (M+1) f(x) G(x)^M$ is a probability distribution of CDF $x \rightarrow 1-G(x)^{M+1}$. We continue the computation with:
\begin{align*}
 \alpha(M,j)&\leq \frac{e^{-j\Delta^2}}{M+1} \int_{-\infty}^{\frac{z-\mu_2 j}{\sqrt{j}}} e^{\sqrt{j}\Delta y}h(y) dy + G_{2,j}(z)^M\\
& \leq \frac{e^{-j\Delta^2}}{M+1} e^{\sqrt{j}\Delta \frac{z-\mu_2 j}{\sqrt{j}}} (1-G\left(\frac{z-\mu_2 j}{\sqrt{j}}\right)^{M+1}) +  G\left(\frac{z-\mu_2 j}{\sqrt{j}}\right)^M  \\
&\leq \frac{e^{-j\Delta^2}}{M+1} e^{\sqrt{j}\Delta \frac{z-\mu_2 j}{\sqrt{j}}} +  G\left(\frac{z-\mu_2 j}{\sqrt{j}}\right)^M \end{align*}

As the inequality is true for all $z \in \mathbb{R}$, it holds that,
\begin{align*}
\forall y \in \mathbb{R}, \ \ \alpha(M, j)&\leq \frac{e^{-j\Delta^2}}{M+1} e^{\sqrt{j}\Delta y} +  G\left(y\right)^M\;.
\end{align*}
Now let $y_M$ be such as $G(y_M)=1-\frac{1}{\sqrt{M}}$. This value ensures that the second term satisfies $G(y_M)^M \leq e^{-\sqrt{M}}=o(M^{-2})$. Observe that $y_M=F^{-1}(\frac{1}{\sqrt{M+1}})$. Using the following equivalent of the quantile function of the normal distribution when the quantile is small (see for instance \cite{Ledford1997}):
$$
F^{-1}(p) = -\sqrt{\log \frac{1}{p^2} - \log \log \frac{1}{p} + \log 2\pi} + o_{p\rightarrow 0}(1)\;,
$$
there exists a constant $ C \in \mathbb{R}$ such that $y_M \leq -C\sqrt{\log M - \log \log M + \log 4\pi}$. This yields 
\begin{align*}
\alpha(M, j) &\leq \frac{e^{-j\Delta^2}}{M+1} e^{-C \sqrt{j}\Delta \sqrt{\log M - \log \log M + \log 4\pi}} + e^{-\sqrt{M}}
\end{align*}

Noting that for all $k \in \mathbb{N}^*$, 
$$
k \log \log M = o(C\sqrt{j} \Delta \sqrt{\log M - \log \log M + \log 4\pi})
$$
we get that $\forall k \in \mathbb{N}^*$:
$$
\alpha(M,j) = o\left(\frac{e^{-j\Delta^2}}{(M+1) (\log M)^k}\right)
$$

This is sufficient to prove that the Gaussian distribution is balanced. Indeed, as for the Bernoulli distribution this term sums as a convergent geometric series in $j$, and with $M=O(t/(\log t)^2)$ we can make the sum in $t$ a convergent Bertrand Series. 

\subsection{The Exponential Distribution is Not Balanced}\label{app:balance_fe_exp_dist}

For $j=1$, a direct calculation yields 
\[\alpha(M,1) = \frac{1}{1+\left(\frac{\mu_1}{\mu_2}\right)M}.\]
Using this, we now prove that the series in Assumption 2. of Theorem~\ref{th::log_regret_sda} is in $\Omega(\log(T))$, hence the balance condition is not satisfied. As all the $\alpha_k(M,j)$ are positive, one can write
\begin{eqnarray*} 
\sum_{t=1}^T \sum_{j=1}^{\lfloor(\log t)^2 \rfloor} \alpha_k(\left\lfloor \beta t/(\log t)^2 \right\rfloor, j) & \geq & \sum_{t=1}^T  \alpha_k(\left\lfloor \beta t/(\log t)^2 \right\rfloor, 1) \\
& = &  \sum_{t=2}^T \frac{1}{1+\left(\frac{\mu_1}{\mu_k}\right)\left\lfloor \beta t/(\log t)^2 \right\rfloor} \\
& \geq & \sum_{t=2}^T \frac{1}{1+\left(\frac{\mu_1}{\mu_k}\right) \beta t/(\log t)^2 } \\
& \geq & C \sum_{t=2}^T \frac{1}{t} = \mathcal{O}(\log(T)),
\end{eqnarray*}
where $C$ is some small enough constant that depend on $\mu_1, \mu_k$ and $\beta$.

%% file: handling_forced_explo.tex
\section{Sketch of Proof with Forced Exploration}\label{app::forced_explo}

In this section, we explain how the proof of Theorem~\ref{th::log_regret_sda} is modified when we add forced exploration with $f_r = \sqrt{\log r}$, that is when in every round $r+1$ we add to $\cA_{r+1}$ every arm $k$ such that $N_k(r) \leq f_r$. It is easy to verify that the proof of Lemma~\ref{lem:dec} remains unchanged, as it is inspired by the analysis of SSMC which also uses forced exploration. We now explain how forced exploration modifies the proof of Lemma~\ref{lemma::decomposition} and how we upper bound the resulting new terms for any exponential family. 

\subsection{Handling Forced Exploration in Lemma~\ref{lemma::decomposition}}\label{app::fe_lemma_dec}

The idea is to use the same proof sketch as without forced exploration. We note $f(r)=\sqrt{\log r}$ the forced exploration rate and $f^{-1}(r)=\exp(r^2)$ its inverse function.

Let us consider the round $a_r = f^{-1}(f(r)-1)$. At this round, the value of exploration function is $f(r)-1=\sqrt{\log r}-1$, which means that the number of pulls of arm 1 is at least $\lfloor \sqrt{\log r}-1 \rfloor\geq \sqrt{\log r}-2$.

Now we look at the length of the interval $r-a_r$:
\begin{align*}
r-a_r = &r - f^{-1}(f(r)-1) \\ 
=& r - \exp((\sqrt{\log r} - 1)^2) \\
=& r - \exp(\log r - 2 \sqrt{\log r} +1)\\
=& r(1-\exp(-2 \sqrt{\log r} +1)) \\
\sim &r \text{ when } r \rightarrow +\infty 
\end{align*}
As $r-a_r$ is equivalent to $r$ when $r$ is large, for any constant $\gamma>0$ there exists a round $r_\gamma$ such that for $r>r_\gamma$: $r-a_r>\gamma r$. This means that after the round $a_r$ arm $1$ faces a linear amount of duels, and has an history of at least $j=\lfloor \sqrt{\log r}-1 \rfloor$ samples. Introducing $b_r$ the random variable giving the first time when $N_1(b_r)=\lfloor \sqrt{\log r}-1 \rfloor$, we necessarily have $b_r\leq a_r$. We now use that $N_1(r)\leq (\log r)^2 \Rightarrow \sum_{j=1}^{\lfloor(\log r)^2\rfloor-1} \tau_j \leq r$, which further implies $$b_r + \sum_{j=\sqrt{\log r}-1}^{\lfloor(\log r)^2\rfloor-1} \tau_j \leq r \Rightarrow \sum_{j=\sqrt{\log r}-1}^{\lfloor(\log r)^2\rfloor-1} \tau_j \leq r-b_r$$

We can then use the same proof as in Appendix~\ref{app::lem_decomposition}: 
\begin{align*}
\bP\left(N_1(r) \leq (\log r)^2\right) &\leq \bP\left( \exists j \in \{\lfloor \sqrt{\log r}-1 \rfloor, ..., \lfloor (\log r)^2 \rfloor-1\}: \tau_j\geq \frac{r-b_r}{(\log r)^2-\lfloor \sqrt{\log r}-1 \rfloor}\} \right)\\
 &\leq \bP\left( \exists j \in \{\lfloor \sqrt{\log r}-1 \rfloor, ..., \lfloor (\log r)^2 \rfloor-1\}: \tau_j\geq \frac{r-a_r}{(\log r)^2-\lfloor \sqrt{\log r}-1 \rfloor}\} \right)\\
&\leq \bP\left( \exists j \in \{\lfloor \sqrt{\log r}-1 \rfloor, ..., \lfloor (\log r)^2 \rfloor-1\}: \tau_j\geq \frac{r-a_r}{(\log r)^2}\} \right) \\
& \leq \bP\left( \exists j \in \{\lfloor \sqrt{\log r}-1 \rfloor, ..., \lfloor (\log r)^2 \rfloor-1\}: \tau_j\geq \frac{\gamma r}{(\log r)^2}\} \right)  
\end{align*}

The constant $\gamma$ does not change the sketch of proof, and we finally have: {\small
\begin{equation}\sum_{r=1}^T\bP(N_1(r))\leq (\log r)^2) \leq r_K' + \sum_{r=r_K'}^T\sum_{j=\lfloor f_r\rfloor - 1}^{(\log r)^2} \left[(K-1)\bP(X_{c_r, c_r, j}<M_{r, j}) + \sum_{k=2}^K \alpha_k(M_{r, j}, j) \right]\label{newlemma}\end{equation}} \hspace{-0.2cm} for any sequence $M_{r,j}$, and a new constant $r_K'$. Observe that the sum in $j$ does not start in 1 as it does in the statement of Lemma~\ref{lemma::decomposition} in the absence of forced exploration. This justifies the introduction of the \textit{generalized balance condition} in Appendix~\ref{app::balance_fe}

\subsection{Exponential Families Satisfy a Generalized Balanced Condition}\label{app::balance_fe}

To conclude the proof as in Theorem~\ref{th::log_regret_sda}, as Random Block Sampling satisfies the Diversity Property, from \eqref{newlemma} (with the choice $M_{r, j} = \left\lfloor \beta r/(\log r)^2 \right\rfloor$) it is sufficient to prove that one-dimensional exponential families satisfy the following generalized balance condition.

\begin{definition}[generalized balance condition] If \SDA{} is defined with a forced exploration rate $f_r$ then the generalized balance condition for the rate $f_r$ is: 
	\[\forall \beta \in (0, 1), \ \ \sum_{r=1}^T \sum_{j=\mathbf{f_r}}^{\lfloor(\log r)^2 \rfloor} \alpha_k(\left\lfloor \beta r/(\log r)^2 \right\rfloor, j) = o(\log T)\;.\]
\end{definition}	 
The following lemma proves that this holds in particular for the choice $f_r = \sqrt{\log(r)}$, which permits to prove that \RBSSDA{} with this forced exploration sequence is asymptotically optimal for distributions that belong to any one-dimensional exponential family.

\begin{lemma}[Generalized balance condition on exponential families]
	If the exploration rate $f_r$ satisfies $\frac{f_r}{\log \log r} \rightarrow +\infty$ then any exponential family of distributions with one parameter satisfies the generalized balance condition.
\end{lemma}

\begin{proof}
	A distribution that belong to a one-dimensional exponential family has a density $f_\theta(y)= f(x,0)e^{\eta(\theta)y-\psi(\theta)}$ for some natural parameter $\theta \in \R$. 
	
	We observe that for any $y_1,...,y_j \in \mathbb{R}^j$, if $\sum_{i=1}^j y_i \leq \mu_k$:
	\begin{align*}
	\prod_{u=1}^j f_{\theta_1}(y_u) = & \prod_{u=1}^j e^{(\eta(\theta_1)- \eta(\theta_k)) y_u-(\psi(\theta_1)- \psi(\theta_k))} f_{\theta_k}(y_u) \leq e^{-jI_1(\mu_k)} \prod_{u=1}^j f_{\theta_k}(y_u)
	\end{align*}
	
	This inequality ensures that for all $ x,u \in \mathbb{R}$, if $F_{k,j}^{-1}(u) \leq \mu_k$:
	\begin{align*}
	F_{1,j}(x) \leq e^{-j I_1(\mu_k)} F_{k,j}(x) \Rightarrow F_{1,j}(F_{k,j}^{-1}(u)) \leq e^{-j I_1(\mu_k)} u
	\end{align*}
	
	So for exponential families a strictly positive gap between two distributions leads to an exponential decrease of the ratio of the CDF of the sum. If we use the fact that for all $u \in \R$:
	\begin{align*} \alpha_k(M, j) &= \int_{-\infty}^{+\infty} f_{1,j}(x)G_{2,j}(x)^M d\bP(x)\\
	&\leq \int_{-\infty}^{u} f_{1,j}(x)G_{2,j}(x)^M + \int_{u}^{+\infty} f_{1,j}(x)G_{2,j}(x)^M d\bP(x)\\
	&\leq F_{1,j}(u)+G_{2,j}(u)^M\\
	\end{align*}
	
	Then $\forall \beta \in (0,1)$ and for all sequence $u_{r}$:
	
	\begin{align*}
	\sum_{r=1}^T \sum_{j=\mathbf{f_r}}^{\lfloor(\log r)^2 \rfloor} \alpha_k(\left\lfloor \beta r/(\log r)^2 \right\rfloor, j) & \leq \sum_{r=1}^T \sum_{j=\mathbf{f_r}}^{\lfloor(\log r)^2 \rfloor} (1-u_{r})^{\left\lfloor \beta r/(\log r)^2 \right\rfloor} + e^{-j I_1(\mu_k)} u_{r}\\
	& \leq \sum_{r=1}^T \log(r)^2(1-u_r)^{\left\lfloor \beta r/(\log r)^2 \right\rfloor} + \sum_{r=1}^T \frac{e^{-f_r I_1(\mu_k)}}{1-e^{-I_1(\mu_k)}}u_r
	\end{align*}
	We now choose $u_r$ of the form $u_r=\frac{(\log r)^k}{r}$. Indeed, for the first term we get: 
	\begin{align*}
	\log(r)^2(1-u_r)^{\left\lfloor \beta r/(\log r)^2 \right\rfloor} & \leq \log(r)^2 \exp\left(-\left\lfloor \beta r/(\log r)^2 \right\rfloor \frac{(\log r)^k}{r}\right) \\
	& \leq \log(r)^2 \exp\left(-(\beta r/(\log r)^2-1) \frac{(\log r)^k}{r}\right)\\
	& \leq \zeta_k (\log r)^2 \exp\left(-\beta (\log r)^{k-2}\right)\\
	& = o(r^{-1}) \text{ for } k>2
	\end{align*}	
	where $\zeta_k$ is an upper bound for $\exp(\frac{(\log r)^k}{r})$. From now on we work with $k=3$.  For the second term we have to study $u_r e^{-I_1(\mu_k)f_r}$:
	\begin{align*}
	u_r e^{-I_1(\mu_k)f_r} &= \exp\left(\log u_r -I_1(\mu_k)f_r\right)\\
	& = \exp\left(3 \log \log r - \log r -I_1(\mu_k)f_r\right)
	\end{align*}
	We see that $u_r e^{-I_1(\mu_k)f_r}=o(r^{-1})$ if $3 \log \log r -I_1(\mu_k)f_r \rightarrow 0$. This condition is achieved if $f_r/\log\log r \rightarrow +\infty$, hence an exploration rate satisfying this condition ensures the generalized balance condition for any exponential family of distributions with one parameter for this rate.
\end{proof}

We point out the fact that this forced exploration is not necessary in \SDA{}, as we proved that some distributions (Bernoulli, Gaussian, Poisson) directly satisfy the balance condition defined in Assumption 2. of Theorem~\ref{th::log_regret_sda}. We leave for future research an in-depth analysis of the properties of different families of distribution that could exhibit general conditions for the use of forced exploration (or not) in the \SDA{} family.